\documentclass{article} 
\usepackage[accepted]{icml2022}
\usepackage{tikz}
\usepackage{times}
\usepackage{comment}
\usepackage{multirow}
\usepackage{array}
\usepackage{caption}
\usepackage{subcaption}
\usepackage{booktabs}
\usepackage{makecell}
\usepackage[super]{nth}
\usepackage{hyperref}
\usepackage{csvsimple}
\usepackage[inline]{enumitem}

\usepackage{amsthm,amsmath,amsfonts,bm,commath,mathtools}

\newtheorem{conjecture}{Conjecture}[section]

\def\eqref#1{equation~\ref{#1}}

\def\1{\bm{1}}

\DeclareMathAlphabet{\mathsfit}{\encodingdefault}{\sfdefault}{m}{sl}
\SetMathAlphabet{\mathsfit}{bold}{\encodingdefault}{\sfdefault}{bx}{n}

\newcommand{\E}{\mathbb{E}}

\newcommand{\R}{\mathbb{R}}

\newcommand{\maj}{\mathrm{maj}}

\DeclareMathOperator*{\argmax}{arg\,max}

\DeclareMathOperator*{\supp}{supp}

\DeclareMathOperator{\sign}{sign}
\DeclareMathOperator{\sgn}{sgn}

\DeclarePairedDelimiterX{\expectarg}[1]{[}{]}{%
  \ifnum\currentgrouptype=16 \else\begingroup\fi
  \activatebar#1
  \ifnum\currentgrouptype=16 \else\endgroup\fi
}

\newcommand{\innermid}{\nonscript\;\delimsize\vert\nonscript\;}
\newcommand{\activatebar}{%
  \begingroup\lccode`\~=`\|
  \lowercase{\endgroup\let~}\innermid 
  \mathcode`|=\string"8000
}
\usepackage{xcolor}
\usepackage{soul}
\usepackage{tcolorbox}
\tcbuselibrary{xparse}

\colorlet{todocolor}{red!50!black}
\colorlet{todobg}{todocolor!30}

\newtcbox{\todobox}{on line, colback=todobg, colframe=todocolor, size=fbox, boxsep=2pt, fontupper=\sffamily\footnotesize, title={[TODO]}, fonttitle=\scriptsize, attach title to upper={~~}, coltitle=todocolor}
\newtcolorbox{todopara}{colback=todobg, colframe=todocolor, size=small, fontupper=\sffamily\footnotesize, title={[TODO]}, fonttitle=\scriptsize, attach title to upper={~~}, coltitle=todocolor}

\usepackage[normalem]{ulem}

\usepackage{hyperref}
\usepackage{url}

\theoremstyle{plain}
\newtheorem{theorem}{Theorem}[section]

\newtheorem{lemma}[theorem]{Lemma}

\theoremstyle{definition}
\newtheorem{definition}[theorem]{Definition}

\theoremstyle{remark}

\title{Coin-Flipping Neural Networks}

\begin{document}

\twocolumn[
\icmltitle{Coin-Flipping Neural Networks}


\icmlauthor{Yuval Sieradzki}{technion}
\icmlauthor{Nitzan Hodos}{technion}
\icmlauthor{Gal Yehuda}{technion}
\icmlauthor{Assaf Schuster}{technion}

\icmlaffiliation{technion}{Department of Computer Science, Technion - Israel Institute of Technology, Haifa, Israel}

\icmlcorrespondingauthor{Yuval Sieradzki}{syuvsier@campus.technion.ac.il}
\icmlcorrespondingauthor{Nitzan Hodos}{hodosnitzan@campus.technion.ac.il}
\icmlcorrespondingauthor{Assaf Schuster}{assaf@technion.ac.il}

\icmlkeywords{Machine Learning, ICML}

\vskip 0.3in
]



\printAffiliationsAndNotice{}  

\begin{abstract}
We show that neural networks with access to randomness can outperform deterministic networks by using amplification. We call such networks Coin-Flipping Neural Networks, or CFNNs.
We show that a CFNN can approximate the indicator of a $d$-dimensional ball to arbitrary accuracy with only 2 layers and $\mathcal{O}(1)$ neurons, where a 2-layer deterministic network was shown to require $\Omega(e^d)$ neurons, an exponential improvement \cite{DBLP:journals/corr/SafranS16}.
We prove a highly non-trivial result, that for almost any classification problem, there exists a trivially simple network that solves it given a sufficiently powerful generator for the network's weights.
Combining these results we conjecture that for most classification problems, there is a CFNN which solves them with higher accuracy or fewer neurons than any deterministic network.
Finally, we verify our proofs experimentally using novel CFNN architectures on CIFAR10 and CIFAR100, reaching an improvement of 9.25\% from the baseline.

\end{abstract}

\section{Introduction}

A fundamental question in computer science is whether randomness can be used as a resource: can an algorithm solve a problem using less time or memory when given the option to flip coins?
While in general this problem has not yet been solved, there are examples of problems for which the answer was shown to be positive.
For example, using randomness, the volume of a convex body in an $n$ dimensional Euclidean space can be approximated to an arbitrary factor in polynomial time \cite{dyer1991random}, but in deterministic polynomial time it is not possible to approximate the volume of a convex body within even a polynomial factor \cite{barany1987computing}.

One important feature of randomized algorithms is \emph{amplification}, where by sampling multiple times from the algorithm and aggregating the results we can increase the probability of success.
For example, for a ground-truth classification function $f:\R\to\{-1,1\}$, assume we are given a randomized classifier $h:\R\to\{-1,1\}$ with an associated probability of success $\Pr[h(x_0)=f(x_0)]=\frac{2}{3}$ for some query point $x_0\in\R$.
We can amplify the probability of correctly guessing $f(x_0)$ by sampling $n$ times from $h(x_0)$ and taking a simple majority; using the well known Hoeffding Inequality we can see that the probability of correct classification is proportional to $1-e^{-n}$, a significant improvement over $\frac{2}{3}$.

The power of amplification for randomized algorithms motivates us to apply it in the context of neural networks. In this work we show that neural networks with access to random values can outperform deterministic networks, even to an exponential factor, given we are allowed to amplify their outputs.
We call this approach Coin-Flipping Neural Networks, and study it in the context of classification problems.

\begin{table}[t]
\caption{\small Accuracy measured for CFNN architectures and their deterministic equivalents, as explained in Section \ref{section:experiments}. The Hypernetwork is comprised of a network $G$ which learns a parameter distribution for a network $N$. ResNet20 was used with Dropout as its source of randomness. Both models are random during inference.
CFNN networks consistently achieve improved accuracy on both CIFAR10 and CIFAR100 datasets.}
\label{tbl:summary}
\vskip 0.15in
\begin{center}
\begin{small}
\begin{sc}
\begin{tabular}{lcccr} 
\toprule
Arch. & Det. & CFNN & $\Delta$ \\
\midrule
\multicolumn{4}{c}{CIFAR10} \\
\midrule

Hypernet & 79.71 $\pm$ 0.5 & 81.45 $\pm$ 0.2 & +1.74$\pm$ 0.6 \\
ResNet20 & 87.82 $\pm$ 0.6 & 90.49 $\pm$ 0.1 & \textbf{+2.67 $\pm$ 0.7} \\
\midrule
\multicolumn{4}{c}{CIFAR100} \\
\midrule
Hypernet & 58.09 $\pm$ 0.2 & 61.24$\pm$ 0.2 & +3.15 $\pm$0.3\\
ResNet20 & 50.90 $\pm$ 0.1 & 60.16 $\pm$ 0.6 & \textbf{+9.25 $\pm$ 0.7} \\
\bottomrule
\end{tabular}
\end{sc}
\end{small}
\end{center}
\vskip -0.1in
\end{table}

\subsection{Related Work}
\label{related_work}
Randomness has been used in deep learning in various contexts.
As described in a survey by \cite{gallicchio2017randomized}, randomness is used for data splitting, data generation, observation order, hyperparameters optimization, and structural random elements.
We address these methods and others in the following paragraphs.
First, we note that amplification by sampling multiple times from a network is rarely used. In fact, most methods aim to reduce the number of samples required from a computational efficiency standpoint.
In this work we argue that sampling multiple times can be beneficial to model accuracy, which is a significant deviation from common practice \cite{DBLP:journals/corr/abs-1906-02691} \cite{doersch2021tutorial}.

\paragraph{Randomness During Initialization}
There is a large number of machine learning algorithms that use randomness to construct their model. From classic algorithms like Random Forest and Random Ensemble Learning, that use random subsets of the input, to more complex models such as RVFL, Randomized Neural Networks and Reservoir Computing, which use a set of randomly sampled function as an expanded input to a learned classifier \cite{DBLP:journals/corr/abs-2002-12287}, \cite{DBLP:journals/corr/abs-1803-03635}.
These algorithms randomly generate a deterministic function; hence, they are not CFNNs.
Another important example are Ensemble methods such as \cite{DBLP:journals/corr/abs-1904-05488} \cite{NEURIPS2020_b86e8d03}, where
multiple sub-networks are trained on randomized subsets of the training data or using regularization to diversify their outputs. These models are also deterministic after initialization, and can be viewed as approximations to CFNNs.

\paragraph{Structural Randomness}
These networks utilize randomness as a part of their computation model. In Spiking Neural Networks \cite{spiking_neural_network}, neurons output timed pulses, modelled after the behavior of biological neurons. These pulses are usually described as a random process, e.g. a Poisson process.
In Restricted Boltzmann Machines and other Energy Based Models, the annealing process is stochastic and is modelled after the cooling process of a crystal lattice (\cite{boltzmann_machines1}, \cite{boltzmann_machines2}). Another example is Stochastic Computation \cite{universal_scnn}.
Such methods don't strictly fall under the CFNN definition given in this paper, which deals with standard DNNs with additional randomization.

\paragraph{Stochastic neural networks}
These are networks for which a neuron's output is a sample from a distribution parameterized by the inputs to the neuron, e.g. \cite{stochastic_tang}, \cite{Neal90learningstochastic}.
These models are examples of networks which have access to randomness. However, amplification is not used by them. 
Also, a stochastic neural network is usually used to allow better data modelling, e.g. \cite{de2019stochastic} which aims to learn a model for a map between distributions (e.g. point cloud data). CFNNs do not attempt to better calculate the distribution of the data, but find a distribution that solves a problem given the data. That distribution might be very different from the actual data distribution.

\paragraph{Confidence Estimation}
A common use of randomness in deep learning is to calculate confidence estimates on the output of NNs. For example, Bayesian NNs [\cite{bayesian_hands_on},\cite{bayesian_tutorial}] are networks whose parameters (or activations) are modeled as probability distributions conditioned on the training data. 
Given a prior over the parameters, Bayes' theorem is applied to get the distribution of the weights.
Another example of confidence estimation is MCDropout, which uses the common Dropout \cite{14original_dropout_paper} during inference by averaging the results of multiple samples from the network's output \cite{gal2016mc-dropout}. The variance is then used to estimate confidence for the output.
CFNNs focus on applying randomness in order to improve accuracy, which is markedly different from the focus of MCDropout and Bayesian NNs.
Confidence estimates can be calculated for CFNNs, but we are not interested in this in the context of this paper.
Another important difference to CFNNs is that Bayesian NNs require the specification of a prior on the parameters of the network. CFNNs don't require such explicitness. Inductive biases in the network's design and the training used can be reformulated as priors; however, in most cases this is difficult to do explicitly.
Also note that we have experimented with MCDropout as a CFNN which improved its accuracy by up to 9.2\% from our baseline. See Section \ref{section:experiments}, Table \ref{tbl:dropout_results}.

\paragraph{Generative Models}
An important class of networks that employ randomness are generators. For example GAN models such as StyleGAN \cite{DBLP:journals/corr/abs-2106-12423}, InfoGAN \cite{DBLP:journals/corr/ChenDHSSA16} and Parallel Wave GAN \cite{yamamoto2020parallel} have proven to be excellent generators of images and audio. Variational Autoencoders \cite{kingma2014autoencoding} \cite{DBLP:journals/corr/abs-1906-02691} \cite{doersch2021tutorial} learn a distribution over a latent space and decode samples from it to generate images with high fidelity. Denoising Diffusion Probabilistic Models \cite{DBLP:journals/corr/abs-2006-11239} model an iterative noising process and learn an inverse operation which allows to sample images directly from noise.
These examples show that neural networks can be trained to learn complex distributions, which could then be used as building blocks for CFNNs. However, except for DDPMs, these methods do not use amplification.
DDPMs can be viewed as using amplification, but they are not a general framework for amplification in NN whereas CFNNs are.

Finally, In \cite{dwaracherla} it was shown that networks with randomness can be used as more effective agents in a reinforcement learning setting. We consider this work an example of CFNNs, since its explicit motivation was to improve the network's performance on a task by utilizing randomness. However they also show a theorem which questions the need for complex use of randomness in the first place. We address this question in Section \ref{section:theorem}.

\subsection{Our Contribution}
In this work we show that by using amplification, CFNNs  can improve accuracy and reduce space complexity (number of neurons) compared to deterministic networks on classification tasks .

We give an example to such a task, the classification of a $d$-dimensional ball around the origin (Section \ref{thm:nd_ball}). Deterministic networks with 2 layers have been shown to require $\Omega(e^d)$ neurons, or $\mathcal{O}(d/\epsilon)$ neurons with 3 layers, in order to solve the task for some approximation error $\epsilon$ \cite{DBLP:journals/corr/SafranS16}. Our CFNN construction requires only $\mathcal{O}(1)$ neurons.
Even when accounting for multiple samples, our CFNN is an exponential improvement over the 2-layer deterministic network in terms of computational complexity.

We prove a theorem dealing with hypernetwork models (Section \ref{section:theorem}), which are comprised of a base network $N(x;w)$ with some parameters sampled by a generator network $w=G(r)$, using a random input $r\in\R^m$. We show that for almost any classification function $f:\R\to\{-1,1\}$, a linear layer $N(x;w)=w\cdot x$ is enough to solve it to arbitrary accuracy, provided $G$ is complex enough.

Using these results, along with theorem 1 from \cite{dwaracherla}, we describe a tradeoff between data complexity and random complexity of possible solutions for a given problem (see Section \ref{section:complexity_continuum} for exact definitions). We conjecture that for any task, there exists a CFNN model that is better (higher accuracy, fewer neurons) than a purely deterministic network (Figure \ref{fig:complexity_continuum}).

Finally, we show experimental evidence for the benefits of CFNNs on CIFAR10 and CIFAR100 datasets, by training a hypernetwork architecture inspired by our proof as well as using Dropout as a part of a CFNN network (Section \ref{section:experiments}). Our networks achieve improvements over the baseline of up to 3.15\% for hypernetworks and 9.2\% for dropout models, see Table \ref{tbl:summary}.

\section{Coin-Flipping Neural Networks}
\label{discussion}

Let $N:\R^d\times \R^m\to[C]$ be a function of two inputs, a data input $x\in\R^d$ and a random input $r\in \R^m$, as well as some possible parameters $w$. Here $C\in\mathbb{N}$ is the number of classes in some classification task and $[C]=\{1,...,C\}$. The function $N$ is implemented as a neural network, with $x$ as its input, where the random value $r$ can be used in several ways, such as: additional random inputs, $N(x,r;w)$; random variables as parameters, $N(x;r,w)$; both, $N(x,r_1;r_2,w)$; and as a part of a hypernetwork, where a generator network $G$ takes a random variable as input and generates parameters for a base network: $N(x;G(r))$.
We will simply write $N(x)$ or $N(x,r)$ for the remainder of the paper.

Since $N(x)$ is a random variable, it has an associated probability function $\mathbf{p}^N:\R^d\to\R^C$. Its $j$-th element is $\mathbf{p}^N_j(x)=\Pr\left[N(x)=j\mid x\right]$. 
We will write $\mathbf{p}(x)$ or $\mathbf{p}$ for the rest of the paper, and use the notation $\mathbf{p}_{f(x)}(x)=\Pr\left[N(x)=f(x)\mid x\right]$ for some function $f:\R^d\to[C]$.

We call $N$ a \emph{Coin-Flipping Neural Network}, or \emph{CFNN}, if it uses amplification during inference.

\subsection{Amplification}
\emph{Amplification} is a method used in randomized algorithms to improve the probability of success, by taking $n$ samples of the algorithm's output and aggregating them using a deterministic function to generate the final result with a higher probability of success. 

The simplest method of aggregating $n$ classification samples is by taking their majority. Under this amplification scheme, the CFNN model is a stochastic classifier with an input-dependent probability function, $\mathbf{p}(x)$.
We can now define:

\begin{definition} (Random Accuracy) \label{def:random_acc}
    Let $\R^d$ be an input space, $\mu$ a distribution on that space, and $C\in\mathbb{N}$.
    Let $f:\R^d\to [C]$ be some ground-truth classification function over $\R^d$. Let $\zeta$ be some distribution over a random input space $R$ and let $n$ be a number of I.I.D samples $r_i\sim\zeta$.
    Given a CFNN $N(x,r)$ with probability function $\mathbf{p}(x)$, the \textbf{\emph{random accuracy}} of $N$ is:
    \begin{align*}
        RA &= \E_{\mu,\zeta}\left[\lim_{n \to \infty} \mathrm{maj}\left\{N(x,r_i)\right\}_{i=1}^n = f(x)\right] \\
        &= \E_{\mu} \left[\argmax_j\left(\mathbf{p}_j(x)\right) = f(x)\right]
    \end{align*}
    Namely, it is the accuracy of the majority taken over an infinite number of samples from $N$. When $RA=1$, we say $N$ \textbf{\emph{classifies $f$ in probability}}.
\end{definition}

Since $N(x)\sim \text{Cat}(C,\mathbf{p})$, the majority approaches $\argmax_j(\mathbf{p}_j)$ with probability 1 as $n\to\infty$.
In experiments we'll use the empirical estimate: 
$\widehat{RA}=\frac{1}{\abs{S}}\sum_{x\in S} \left[\argmax_j\left(\hat{\mathbf{p}}_j(x)\right)=f(x)\right]$, where $n$ is a finite number of samples, $S$ is the dataset of inputs sampled from $\mu$, and $\hat{\mathbf{p}}$ is an empirical estimate of $\mathbf{p}$.

In this paper we claim that by using amplification, a CFNN can have much greater random accuracy than deterministic networks.
Informally, random accuracy measures the ratio of inputs for which a CFNN's likeliest output, i.e. $\argmax_j(\mathbf{p}_j(x))$, is the correct one.
This implies that the network is allowed to make many mistakes, as long as most of the time it is correct.
By using amplification, we could then recover the correct output as the majority of a large enough number of samples.

In our view, this removes a constraint on the search space of deterministic networks, which are forced to always give the same output by definition.
Instead, during training we explicitly allow our networks to be wrong sometimes. The only constraint is placed on the aggregate of several samples, rather than on individual samples themselves. This then allows the network to explore more complex, and random, decision boundaries on the input space.

\section{Classification of a $d$-Dimensional Ball}
\label{thm:nd_ball}
Let $f_R:\R^d\to\{-1,1\},f_R(x)=\sgn({\norm{x}_2 - R})$ be the indicator function for an $d$-dimensional ball of radius $R$.
In \cite{DBLP:journals/corr/SafranS16} it was shown that $f_R$ cannot be approximated by a neural network of depth 2 to a better accuracy than $\mathcal{O}(1/d^4)$ unless its width is $\Omega(e^d)$, but can be approximated to accuracy $\epsilon$ by a neural network of depth 3 that has $\mathcal{O}(d/\epsilon)$ neurons.

We will now show that using a CFNN, we can classify $f_R$ in probability with a 2-layer network whose number of neurons is only $\mathcal{O}(1)$.

This problem is not linearly separable, i.e. there is no deterministic linear classifier which classifies $f_R$.
However, we can find a distribution of linear classifiers which classifies $f_R$ in probability:
\begin{theorem} \label{thm:dball_classification}
    $\exists \gamma$ a distribution of linear classifiers $h:\R^d\to \{-1,1\}$ such that $h$ classifies $f_R$ in probability.
\end{theorem}
This distribution will then be implemented as a CFNN.
We have found two such distributions, and present the simpler one here. The other, which has a nice geometric interpretation, is presented in Appendix \ref{appendix:tangents_ball}.

\begin{proof}

A distribution of classifiers $h:\R^d\to\{-1,1\}$ classifies $f_R$ in probability if for $\norm{x}_2>R$, $\mathbf{p}^h_1(x)=\Pr[h(x)=1\mid x]>\frac{1}{2}$  and $\mathbf{p}^h_1(x)<\frac{1}{2}$ for $\norm{x}_2<R$.

Let $u\sim\mathcal{N}(0,I_d)$ be a vector sampled from the normal distribution on $\R^d$. 
Let $h(x;u,b) = \sgn(u^Tx-b)$ be a linear classifier with $b$ a constant parameter.
The probability of the classifier to output "1", i.e. classify a point as "outside" of the ball, is $\mathbf{p}^h_1(x) = \Pr\left[u^Tx>b\mid x\right] = 1 - \Phi(\frac{b}{\norm{x}_2})$, where $\Phi$ is the standard normal CDF.
If $b>0$, then $\frac{b}{\norm{x}_2} > 0$ and $\Phi(\frac{b}{\norm{x}_2}) > \frac{1}{2}$; hence $\mathbf{p}^h_1(x)<\frac{1}{2}$ for all $x$. If $b<0$, then $\Phi(\frac{b}{\norm{x}_2}) < \frac{1}{2}$ and $\mathbf{p}^h_1(x)>\frac{1}{2}$. Thus, $h$ cannot classify $f_R$ in probability.

\begin{figure}[t]
\vskip 0.2in
\begin{center}
\centerline{\includegraphics[width=0.9\linewidth]{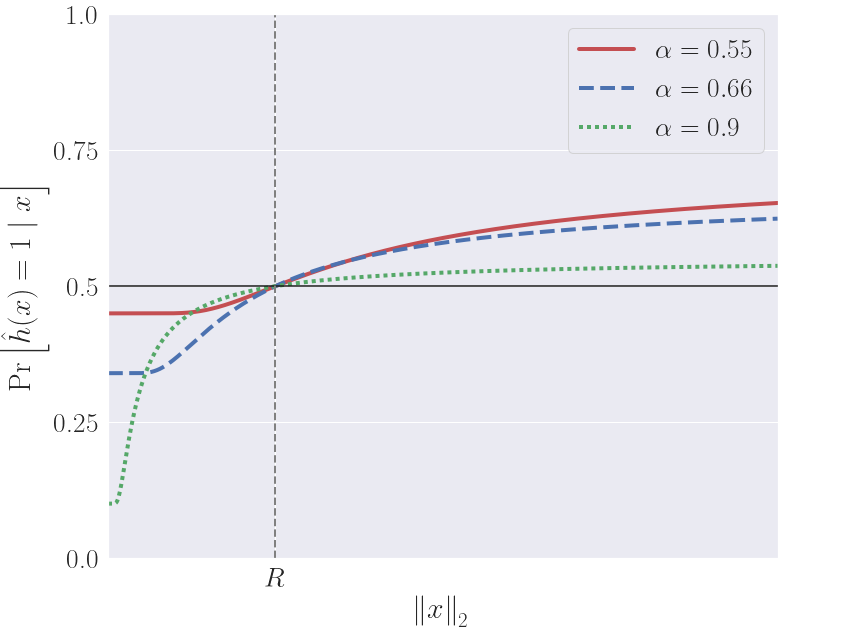}}
\caption{\small The probability function $\Pr\left[\hat{h}(x)=1\mid x\right]$ to classify a point as outside of the ball of radius $R$, displayed as a function of $\norm{x}_2$.
As can be seen, the probability function is greater than $1/2$ for all $\norm{x}_2>R$ and less than $1/2$ for all $\norm{x}_2<1/2$, i.e. $\hat{h}$ classifies the function $f_R(x)=\sign\left(\norm{x}_2-R\right)$ in probability. The graphs of three values of $\alpha$ are displayed.}
\label{fig:prob_h_dball}
\end{center}
\vskip -0.2in
\end{figure}

We can correct this with another variable $t\sim Ber(\alpha)$, with $\alpha>\frac{1}{2}$. If $t=1$, we will use $h$ above; if $t=0$, we will use a constant classifier (which is also linear) to output "1".
We thus define $\hat{h}(x;u,t, b) = \sgn(t (u\cdot x - b - 1) + 1)$; its probability to output "1" is $\mathbf{p}^{\hat{h}}_1(x)=(1-\alpha) + \alpha\mathbf{p}^h_1(x)=1-\alpha\Phi(\frac{b}{\norm{x}_2})$.
We want this probability to be $<\frac{1}{2}$ if $\norm{x}_2<R$ and $>\frac{1}{2}$ if $\norm{x}_2>R$. As $\Phi$ is continuous and strictly monotonic, this implies $1-\alpha\Phi(\frac{b}{R})=\frac{1}{2}$, which gives the parameter value $b=R\cdot\Phi^{-1}\left(\frac{1}{2\alpha}\right)$. 
Since for all values of $(u,t)$, $\hat{h}$ is a linear classifier as a function of $x$, we proved that $\hat{h}$ is a random linear classifier that classifies $f_R$ in probability.
\end{proof}

Note $\hat{h}$ can be implemented as a 2-layer CFNN, with layer 1 as $o_1(x;u,b) = u\cdot x - b - 1$ and layer two as $o_2(o_1;t) = \sgn(t \cdot o_1 + 1)$.
The number of neurons in this network is 2, one for each layer.

\paragraph{Computational Complexity} We have shown a CFNN with $\mathcal{O}(1)$ neurons, which is an exponential improvement in space complexity w.r.t deterministic networks.
In terms of computational complexity, which in our context is counted as the number of times a neuron performs a computation, the improvement is still exponential. As this discussion is very nuanced, we delay it to Appendix \ref{appendix:comp_safran}, but bring the main results here.

In \cite{DBLP:journals/corr/SafranS16} the accuracy metric used was MSE: $\int_{\R^d}\left(f_R(x)-\hat{h}(x)\right)^2\mu(x)dx = \norm{f_R-\hat{h}}_{L_2(\mu)}^2$, where $x\sim\mu$.
Our network is a random variable, so a natural extension is to calculate the MSE over the network's probability as well, $\norm{f_R-\hat{h}}_{L_2(\mu\times\gamma)}^2$, where $\gamma$ represents the distribution of network realizations $\hat{h}$.
We state the following theorem:
\begin{theorem} \label{thm:any_mu_exists_n}
    Let $\{\hat{h}_i\}_{i=1}^n$ be $n$ IID realizations of $\hat{h}$, and its majority $\maj\{\hat{h}_i\}$.
    For any input distribution $\mu$ with $\int_{\mathbb{S}^{d-1}(R)}d\mu=0$ and any $\epsilon>0$, $\exists n\in\mathbb{N}$ such that $\norm{f_R-\maj\{\hat{h}_i\}}^2_{L_2(\mu\times\gamma)}\leq\epsilon$.
\end{theorem}
The proof is given in \ref{appendix:comp_safran}.
The theorem shows that the majority of $\hat{h}$ can approximate the $d$-dimensional ball to arbitrary accuracy even in the MSE case. The only caveat is that the distribution of inputs cannot assign any probability to the surface of the ball. Since we can arbitrarily assign any value to the surface of the ball without changing the definition of the original classification function, this restriction isn't significant.

Given Theorem \ref{thm:any_mu_exists_n} we also show that the computational complexity of $\maj\{\hat{h}_i\}$, i.e. the number of times we sample from the network times its size, is only $\Omega(poly(d))$. This requires a stronger restriction on $\mu$, specifically that it does not assign a probability greater than $\epsilon/2$ to an exponentially-thin spherical shell around the surface of the ball: $\int_{\norm{x}_2\in[R-z,R+z]}\mu(x)dx > \epsilon/2$ for $z=e^{-d/2}$.

Finally, we show that the specific input distribution used by \cite{DBLP:journals/corr/SafranS16}, $\nu$, satisfies this restriction, hence $\hat{h}$ exponentially improves the computational complexity in approximating the $d$-dimensional ball, compared with deterministic networks.

\section{CFNNs with Strong Generators}
\label{section:theorem}

\begin{figure*}[h]
\vskip 0.2in
\begin{center}
    \begin{subfigure}[b]{0.35\textwidth}
         \centerline{\includegraphics[width=5cm]{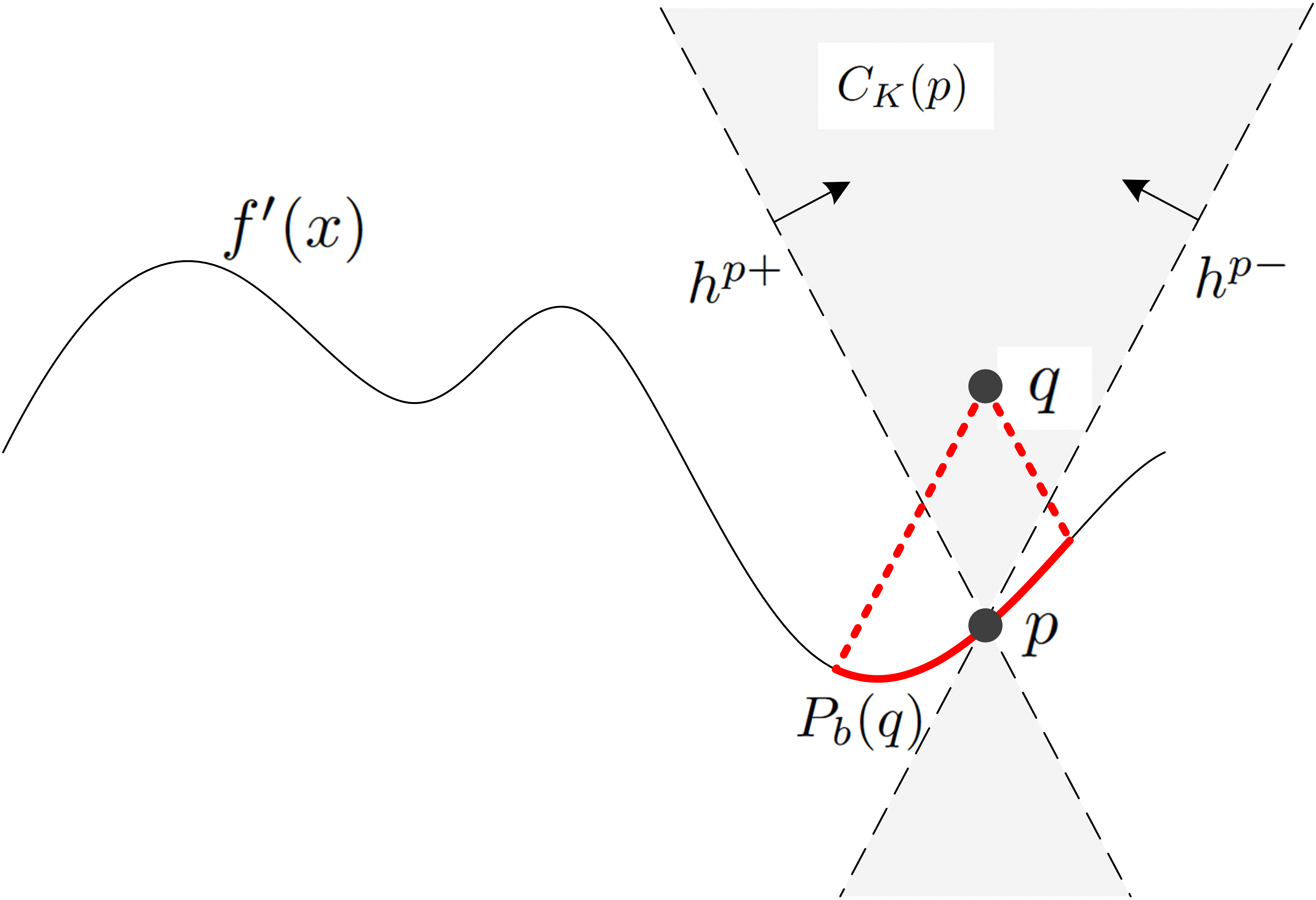}}
        \caption{}
        \label{fig:cone_viz:pb}
    \end{subfigure}
    \begin{subfigure}[b]{0.35\textwidth}
         \centerline{\includegraphics[width=5cm]{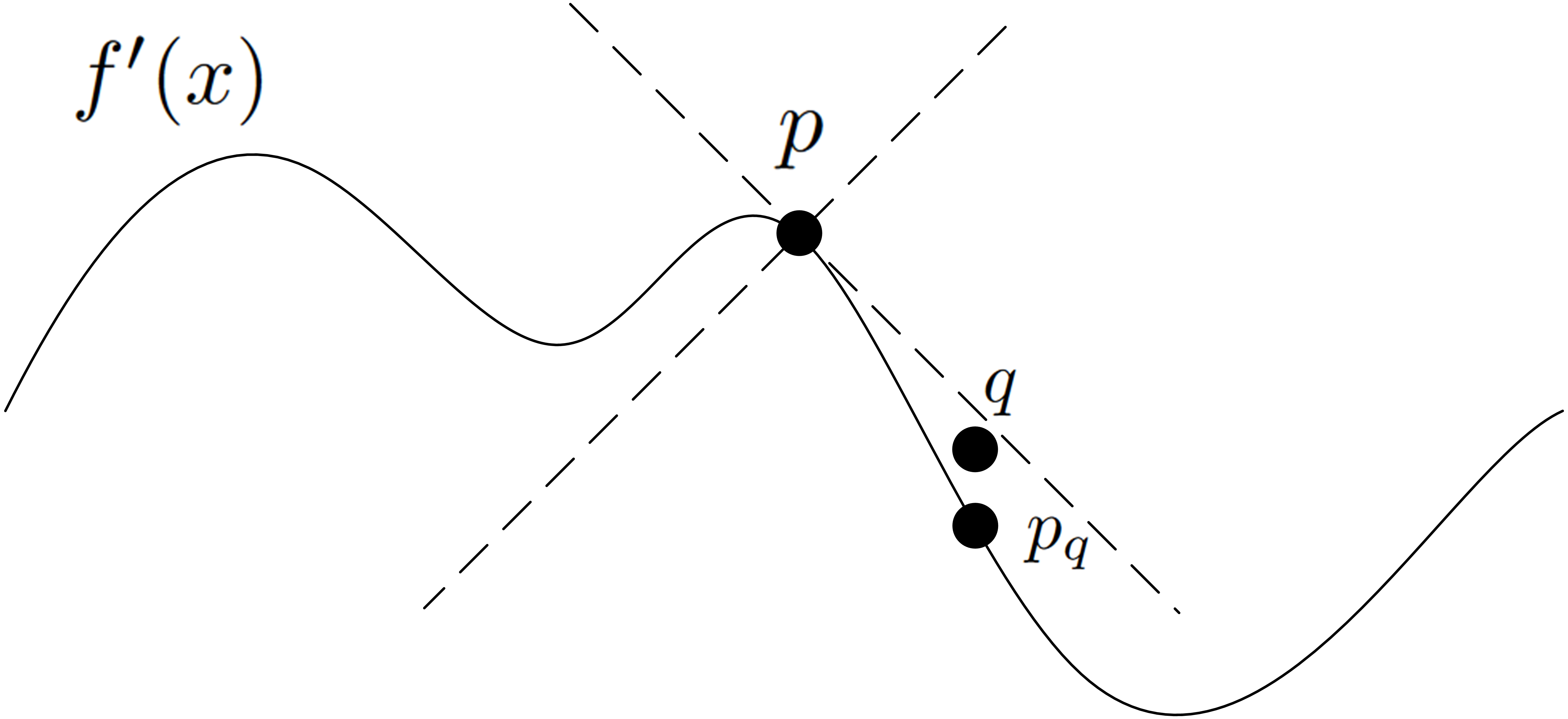}}
        \caption{}
        \label{fig:cone_viz:nonk}
    \end{subfigure}
    \caption{\small (a) Given a point $p$ on the graph of $f'(x)$ we construct two linear classifiers, $h^{p+}$ and $h^{p-}$ of slope $\pm K$. Their normals are illustrated to show which half plane is classified as "1". The set of points on which these two classifiers agree is a cone $C_K(p)$ whose origin is $p$. Given a query point $q$ above teh graph of $f'(x)$, the set of points on the graph whose cone $C_K(p)$ contains $q$ is marked in red and labeled $P_b(q)$. The classifiers $h^{p+},h^{p-}$ for these points agree that $q$ is above the graph. The points on the rest of the graph are the set $P_o(q)$. This means that by integrating over a distribution of points on the graph, points in $P_o(q)$ will contribute $\frac{1}{2}$ and points in $P_b(q)$ will contribute $1$ to the probability of correctly classifying $q$. (b) An illustration that no $p$ whose cone contains $q$, is above $q$. If for some point $p$ on the graph of $f'(x)$ the cone $C_K(p)$ contains $q$, and $p$ is above $q$, then $C_K(p)$ also contains $p_q$ a point on the graph of $f'$, is not K-Lipschitz.} 
    \label{fig:cone_viz}
\end{center}
\vskip -0.2in
\end{figure*}

\cite{dwaracherla} prove that a hypernetwork $N(x;G(r))$, where $G(r)=Ar$ is a linear generator of parameters for a deep base network $N$ with RELU activations, can approximate any distribution over functions.
In other words, they show that if we use a "strong" neural network, then a "weak", i.e. linear, generator is enough to estimate any function.
This leads the authors to ask whether hypernetworks with deeper generators are ever required.
We answer in the positive, by showing a complementary theorem, where the generator is strong and the base network is a linear classifier $h(x;w,b)=\sign(wx-b)$.
\begin{theorem}\label{thm:strong_gens}
    Given almost any function $f:\R\to\{-1,1\}$ and distribution $\zeta$ on $\R$ with $\supp\{f\}\subseteq\supp\{\zeta\}$, there is a network $G:\R\to\R^2$ such that $\forall x \in \R:\Pr_{r\sim \zeta}\left[h(x;G(r))=f(x) \mid x\right] > \frac{1}{2}$.
\end{theorem}

The theorem states two facts: first, for any such classification function $f$, there is a distribution $\gamma$ of linear classifiers which classifies $f$ in probability; second, $\gamma$ can be approximated to arbitrary accuracy using a neural network $G$.
In Appendix \ref{appendix:full_universality} we prove a version of this theorem for functions over $\R^d$, replacing $h$ with a 2-layer network of width $2d$.

\begin{definition} (Separates in Probability)
Given a function $f:\R\to\R$ and a distribution $\gamma$ of linear classifiers, we say that $h\sim\gamma$ \textbf{\emph{separates $f$ in probability}} if 
$\forall x,y\in\R: \Pr_{h\sim\gamma}\left[h(x,y)=\sgn(f(x)-y)\mid x,y\right] > 0.5$ and 
$\Pr_{h\sim\gamma}\left[h(x,f(x))=1\mid x\right] = 0.5$.
\end{definition}
In other words, if $\gamma$ can correctly answer the question \emph{is $y > f(x)$} with probability $>0.5$ we say \emph{$\gamma$ separates $f$ in probability}.
Note that $h(x,y)$ is a linear classifier in $\R^2$.

Informally, the proof of Theorem \ref{thm:strong_gens} is based on finding a $K$-Lipschitz function $f'$ such that $\sign(f'(x))=f(x)$, whose existence is the only limit on $f$. $f'$ is then shown to be separable in probability by a distribution $\gamma$ of linear classifiers.
Since we can separate $f'$, given a point $x\in\R$ we can calculate $h(x,0)$, which is 1 if $f(x)=1$ with probability greater than $\frac{1}{2}$. We thus have that $\gamma$ classifies $f$ in probability.
The distribution $\gamma$ is then described as a continuous function $\gamma(z):\R\to\R^2$, hence there is a neural network $G$ which approximates it to arbitrary accuracy.
Finally we have that $h(x;G(z))$ classifies $f$ in probability.
The full proof is presented in Appendix \ref{appendix:full_universality}. 

\paragraph{The constraint on $f$} $f$ is required to have an associated K-Lipschitz function $f'$ with $\sign(f'(x))=f(x)$. This is a minor limitation in practice; for example, one could take as $f'$ the Fourier series approximation to $f$ of a large enough order, such that the sign requirement is met. The value of $K$ is then determined by the approximation found, since a Fourier series is Lipschitz.
The existence of $f'$ has some other implications which are detailed in Appendix \ref{appendix:constraint_on_f}.

We now prove the first step in the proof of Theorem \ref{thm:strong_gens}.
\begin{theorem}\label{thm:sep-1d}
Let $f':\R\to\R$ be a $K-Lipschitz$ function. Then there exists a distribution $\gamma$ of linear classifiers that separates $f$ in probability.
\end{theorem}

\begin{proof}
Given a point $p=(x_p,y_p)\in\R^2$, and the constant $K$, let $h^{p+}(x,y)=(K,1)\cdot(x,y)-\left(Kx_p+y_p\right)$ and $h^{p-}(x,y)=(-K,1)\cdot(x,y)-\left(-Kx_p+y_p\right)$ be two linear classifiers passing through $p$.
The set $C_K(p)=\left\{(x,y)\mid h^{p+}(x,y)=h^{p-}(x,y)\right\}$ is a cone of slope $K$ whose origin is $p$, as shown in Figure \ref{fig:cone_viz:pb}.
Also, define $h^p(x,y)$ as the linear classifier generated by choosing one of $h^{p+},h^{p-}$ with equal probability. Now by construction, we have:

\begin{lemma} \label{lemma:kcone}
    Let $p=(t,f'(t))$ for some $t\in\R$. Points $q=(x,y)\in C_K(p)$ have $\Pr\left[h^p(q)=\sign(f'(x)-y)\mid q\right]=1$, and points $q=(x,y)\notin C_K(p)$ have $\Pr\left[h^p(q)=\sign(f'(x)-y)\mid q\right]=1/2$.
\end{lemma}

Given a distribution $\zeta$ of points from the support of $f'$, we define $\gamma$ as a sampling procedure:
Sample an input point $t\sim \zeta$. Then, calculate $p=(t,f'(t))$. Finally, sample $h^p(x,y)$. The distribution of the points $p$ is denoted $\zeta_p$.

We now verify that $\gamma$ indeed separates $f$ with probability.
Let $q = (x,y)$ be a point in $\R^2$ such that $y > f'(x)$ w.l.o.g. and let $p_q=(x,f'(x))$.
Let $p=(t,f'(t))$ be a point  on the graph of $f'$ such that $q\in C_K(p)$. Since $f'$ is K-Lipschitz and $f'(x)<y$, then invariably also $f'(t)<y$. Otherwise, $p_q$ would be in $C_K(p)$ as well, which is impossible for K-Lipschitz functions (see Figure \ref{fig:cone_viz:nonk}).
Hence, all such points $p$ with $q\in C_K(p)$ have $f'(t)<y$, and $\Pr[h^p(q)=1\mid q] = 1$ from Lemma \ref{lemma:kcone}. Denote this set $P_b(q)$.
For any point $p$ on the graph of $f'$ such that $q \notin C_K(p)$, we have $\Pr[h^p(q)=1\mid q] = \frac{1}{2}$. Denote the set of these points as $P_o(q)$.
This implies the following:
\begin{align*}
    \Pr_{h\sim\gamma}&[h(q)=1\mid q] = \int_{\R^2}\Pr[h^p(q)=1\mid q]d\zeta_p(p)=\\
    &= \int_{P_o(q)}\underbrace{\Pr_{h\sim C_K(p)}[h^p(q)=1\mid q]}_{=1/2}d\zeta_p(p) \\
    &+ \int_{P_b(q)}\underbrace{\Pr_{h\sim C_K(p)}[h^p(q)=1\mid q]}_{=1}d\zeta_p(p) \\
    &= \frac{1}{2} \int_{P_o(q)}d\zeta_p(p) + \int_{P_b(q)}d\zeta_p(p) > \frac{1}{2}
\end{align*}
The last inequality holds as long as $\int_{P_b(q)}d\zeta(p) > 0$ for all $q\in\R^2$, which is a requirement on $\zeta$. This requirement easily holds for distributions with support over the whole support of $f$, i.e. $\supp\{f\}\subseteq\supp\{\zeta\}$. For example, if $\supp\{f\}=\R$ then $\zeta=\mathcal{N}(0,1)$ is sufficient.
We got $\Pr_{h\sim\gamma}[h(q)=1\mid q] > \frac{1}{2}$ as required, since $y > f(x)$.
\end{proof}

\paragraph{Computing $f'$ or sampling $p$} \label{section:comput_or_sample}
In our proof for Theorem \ref{thm:strong_gens}, we assumed we can compute the function $f'$ to sample a point on its surface.
This explains the source of power of our theorem: the function is encoded in the parameter distribution computed by $G$. 
However, the complexity of $G$ is the same as that of a deterministic network approximating $f$.
If it were easier to sample $f'$ directly than to compute it, we expect the network $G$ to be more efficient.
For example, $\hat{h}$ from Section \ref{thm:nd_ball} uses an easy to sample distribution to drastically reduce the size of the network.

\subsection{Data Complexity vs. Random Complexity Tradeoff}
\label{section:complexity_continuum}

\begin{figure}[ht]
\vskip 0.2in
\begin{center}
\centerline{\includegraphics[width=5.5cm]{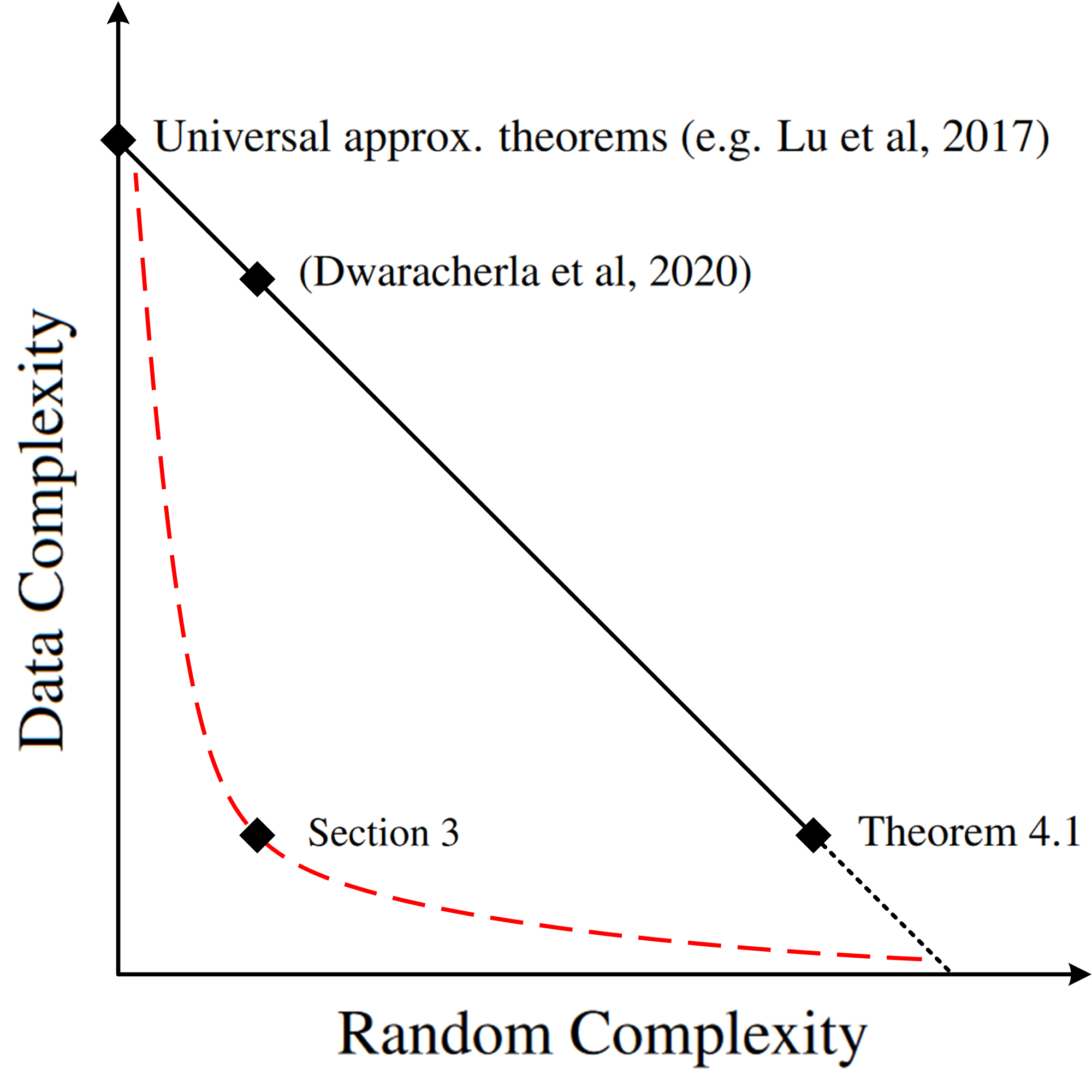}}
\caption{\small The complexity landscape of CFNNs that solve classification problems. The network in \cite{dwaracherla} has a high $DC$ and low $RC$; the network in Theorem \ref{thm:strong_gens} has low $DC$ and high $RC$. It is logical that on a line of constant complexity $DC+RC$, there are many networks that solve the same problem. $\hat{h}$ from Section \ref{thm:nd_ball} has $DC+RC\in\mathcal{O}(1)$ which is exponentially better than deterministic networks. We thus conjecture that networks with minimal complexity must be CFNN, i.e. they exist on the dashed red line.}
\label{fig:complexity_continuum}
\end{center}
\vskip -0.2in
\end{figure}

We now connect Theorem \ref{thm:strong_gens} and Theorem 1 in \cite{dwaracherla}. Their work used a hypernetwork with a very simple (linear) generator of weights to a potentially highly complex base network. Our theorem uses a very simple base network with a potentially highly complex generator of its weights. In both cases, the networks can achieve arbitrary accuracy.

Define the Data Complexity of a CFNN $DC\left(N(x, r)\right)$ as the number of neurons in the computational path from data $x$ to its output. Also, define the Random Complexity $RC\left(N(x,r)\right)$ as the number of neurons in the computational path from $r$ to its output.
It is only logical to suggest that we can find networks with different trade-offs of $DC(N)$ and $RC(N)$, and as long as $DC(N)+RC(N)$ is large enough the network could achieve the same accuracy as the networks in the theorems. These networks form a line of constant $DC(N)+RC(N)$, shown in Figure \ref{fig:complexity_continuum}.

However, in Section \ref{thm:nd_ball} we've given an example of a CFNN with $DC+RC\in\mathcal{O}(1)$, that can solve a problem which deterministic NNs require exponential DC to solve.
This example is not on the line, which leads us to conjecture the following:

\begin{conjecture}
    Given a classification function $f$, and any deterministic network $g$ which approximates it, there exists a CFNN $h$ with $DC(g)\in\Omega(DC(h)+RC(h))$ which approximates $f$ to at least the same accuracy as $g$.
\end{conjecture}
In other words, we conjecture that for any classification problem, the tradeoff $DC+RC$ has a minimum that can only be achieved by CFNN networks.

\section {Experimental Study}
\label{section:experiments}
In this section we experiment with CFNNs as classifiers of CIFAR10 and CIFAR100 \cite{cifar}. We present 3 experiments: a study of a hypernetwork architecture inspired by Theorem \ref{thm:strong_gens}; ResNet networks with Dropout viewed as CFNNs; and an analysis of CFNN's accuracy as the number of amplification samples changes.
In all experiments, empirical random accuracy was used as the target metric (see Definition \ref{def:random_acc}).

\label{section:training}
\paragraph{Training} 
In order to optimize the Random Accuracy of a CFNN model using SGD, we need to estimate the probability function $\mathbf{p}$ in a differentiable way.
Then, the loss can be any standard loss function for classification, calculated as $\text{Loss}(\hat{\mathbf{p}}(x),y)$, e.g. Cross Entropy: $\text{CEL}(\hat{\mathbf{p}},y) = -\log(\hat{\mathbf{p}}_y)$.

We estimate the probability function as an approximate histogram: $\Tilde{\mathbf{p}}(x) = \frac{1}{n} \sum_{j=1}^n \Tilde{I}\left(N(x,r_j)\right)$. $\Tilde{I}$ is an approximate, differentiable one-hot vector such that it is almost 0 everywhere except at the index $\argmax_i N(x,r_i)$ where it is close to 1, and the sum of its elements is 1.
We use the Gumbel Softmax function \cite{gumbel_softmax}, as $\Tilde{I}$: $\mathrm{gsm}(o;\tau)=\mathrm{softmax}(o\cdot \frac{1}{\tau})$.
In practice we used $\tau=1$ for most experiments, i.e. $\Tilde{I}(o)=\mathrm{softmax}(o)$.

The training itself is standard SGD, where for each batch of images we sample $n$ instances of the random variable and calculate $\Tilde{p}(x)$.

\subsection{HyperNetowrk CFNNs}
\label{expr:hypernet}

\begin{figure}
\vskip 0.2in
\begin{center}
\centerline{\includegraphics[width=6cm]{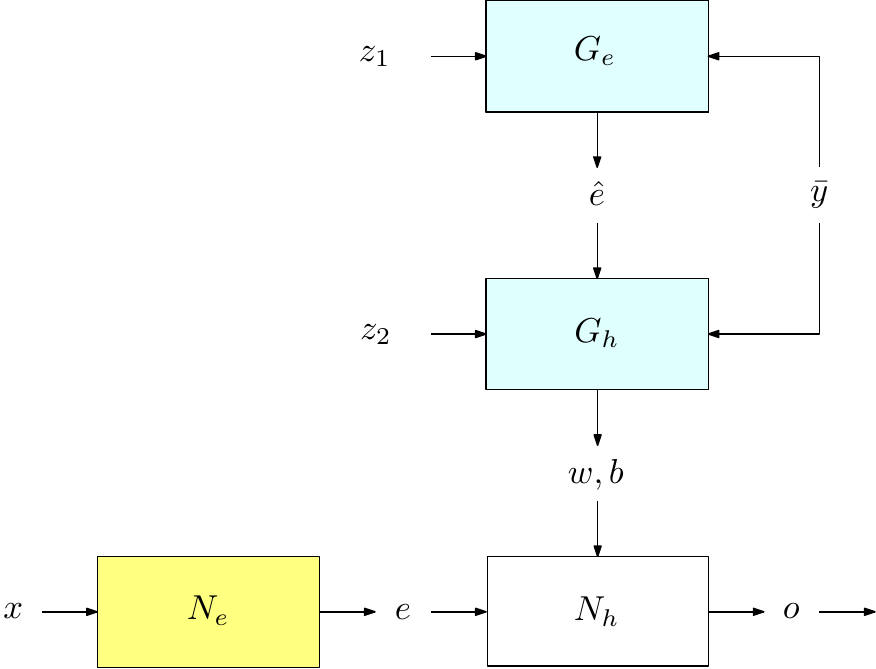}}
\caption[asd]{
    \small Hypernetwork CFNN architecture, following the two-stage sampling used in the theoretical proof of Theorem \ref{thm:strong_gens}.
    $N_e$ projects the input onto $\mathbb{R}^E$. $G_e$ generates $\hat{e}$ using a condition $\bar{y}\in[C]$ sampled using the class frequencies of the data. $G_h$ uses $\hat{e},\bar{y}$ to directly compute the  parameters of a linear layer $N_h(x;w,b)=wx-b$. Both generators receive an additional random input $z_1,z_2$ sampled uniformly in $[0,1]^m$.
}
\label{fig:hypernetwork_cfnn}
\end{center}
\vskip -0.2in
\end{figure}

Inspired by Theorem \ref{thm:strong_gens}, we construct our CFNN as a hypernetwork. 
The target function $f:\R^d\to[C]$ has $C$ classes.
In the proof, the generator samples a point $p=(t,f'(t))$ and then computes parameters for a linear classifier $w,b$.
Instead, we first sample  $\bar{y}\in[C]$ directly and generate a candidate $t$ such that $f(t)$ is likely to be $\bar{y}$, using a network $G_e$. 
A network $G_h$ then finds parameters for a linear classifier $N_h$, without explicitly finding $f'$.

To avoid sampling linear classifiers in pixel space directly, which could require the generator to learn a highly modal distribution, we embed the input images into a smaller space $\R^E$ using a deterministic network $N_e:\mathbb{R}^d\rightarrow \mathbb{R}^E$. We've used $E=256$.
Both generators have 6 ResNet-based \cite{resnet} blocks, and $N_e$ has $D$ blocks. 
The final architecture is visualized in Figure \ref{fig:hypernetwork_cfnn} and described in detail in Appendix \ref{appendix:hypernet_details}.

As a baseline, we take the same architecture but remove the generators and replace $N_h$ with a standard FC layer which has learned parameters. The difference in accuracy between our network and the baseline can then be attributed to the generator's learned weight distribution.

\begin{table}[t]
\caption{\small Model Random Accuracy with different $N_e$ depths for the CIFAR10 and CIFAR100 benchmarks. Standard deviation presented for depths 6 and 10 respectively.}
\label{tbl:cifar10-table}
\vskip 0.15in
\begin{center}
\begin{small}
\begin{sc}
\begin{tabular}{l|cccr}
\toprule
&\multicolumn{3}{c}{CIFAR10} \\
$D$ & CFNN & Det. & $\Delta$ \\
\midrule
2 & 44.05 & 42.14 & +1.91 \\
\textbf{6} & \textbf{81.45 $\pm$ 0.24} & \textbf{79.71 $\pm$ 0.52} & \textbf{+1.74$\pm$ 0.57} \\
10 & 89.37 & 88.41 &+0.96\\
14 & 91.81 & 91.42&+0.39 
\\
20 & 92.37 & 92.89  &-0.52  
\\
\midrule
&\multicolumn{3}{c}{CIFAR100}\\
\midrule
6 & 45.29 & 44.23 & +1.06\\
\textbf{10} & \textbf{61.24$\pm$ 0.41} & \textbf{58.09 $\pm$0.22} & \textbf{+3.15 $\pm$0.29}\\
14 & 66.48&	65.16	&+1.32 \\
20 & 69.26&	67.68	&+1.58 \\
56 & 72.00&	72.23&	-0.23 \\ 
\bottomrule
\end{tabular}
\end{sc}
\end{small}
\end{center}
\vskip -0.1in
\end{table}

We hypothesize that when $N_e$ is relatively simple, the problem remains complex enough in the embedding space that the network will have more to gain by using majority. On the other hand, deep $N_e$ might oversimplify the problem making it easily separable with a single deterministic linear classifier, which would result in negligible difference between the CFNN and the baseline.

Results are shown in Table \ref{tbl:cifar10-table}. The hypernetwork was able to improve accuracy by up to 3.15\% on CIFAR100 relative to the baseline, for a relatively shallow depth of 6 and 10 blocks respectively. As expected, with increasing $L$ the improvement decreases.
The improvement is more pronounced on CIFAR100 due to the number of classes. For a CFNN to be correct, it only requires the correct answer to be slightly more likely than the rest.
With more classes, this probability can be much lower, which would allow the CFNN to be "more random" in its output.
This suggests that for more difficult datasets, the improvements could be greater.
We elaborate on this idea further in appendix \ref{appendix:delta}, and present further ablations in \ref{appendix:hypernet_ablations}.

\subsection{Dropout CFNNs}
\label{expr:dropout}
We now show that the CFNN framework can be easily applied to standard networks that already use randomness during training, e.g. networks with Dropout \cite{14original_dropout_paper}.
We use a ResNet20 \cite{resnet} model with Dropout and train it using standard techniques and as a CFNN.
During inference, we simply use the Dropout layer as we do during training.

We also perform an ablation of our technique.
We measure the performance gained using standard training with/out amplification during inference. 
We applied standard inference to a model trained with majority amplification which is equivalent to sampling the expected mean of the Dropout mask. Lastly, we also measure a different amplification scheme, the mean of model outputs, during inference. 
The output using the empirical mean is simply $\argmax\{\frac{1}{n}\sum N_i(x)\}$.
Note that standard training with mean amplification during inference is equivalent to MCDropout, a method used to estimate confidence intervals \cite{gal2016mc-dropout}. We further address this method in Appendix \ref{appendix:mcdropout}.

Results are shown in Table \ref{tbl:dropout_results}.
Performing amplification on ResNet20 with Dropout improved performance even when using standard training, but the greatest improvement was achieved when amplification was used in training. 
The fact that the MajTrain/StandardTest model performed so poorly shows that most of the accuracy is gained from the variance of the distribution and not its mean.
This result, as well as the nearly 10\% magnitude of the accuracy gain for CIFAR100, validates our technique.
Importantly, we see that the CFNN framework can be effectively used on standard models with very few changes.

\begin{table}[t]
\caption{\small Random Accuracy of ResNet20 with dropout on CIFAR10 and CIFAR100. Standard deviation is measured as the deviation of the mean of 3 training runs, and 3 inference measurements for each one.}
\label{tbl:dropout_results}
\vskip 0.15in
\begin{center}
\begin{small}
\begin{sc}
\begin{tabular}{lccr}
\toprule
method & CIFAR10 & CIFAR100 \\
\midrule
\makecell[l]{Standard Train,\\ Standard Test} & 87.82 $\pm$ 0.65 & 50.90 $\pm$ 0.03 \\
\midrule
\makecell[l]{Standard Train,\\ Mean in Test \\ (MCDropout)} & 88.61 $\pm$ 0.10 & 53.66 $\pm$ 0.24 \\
\midrule
\makecell[l]{Standard Train, \\Maj in Test} & 88.37 $\pm$ 0.18 & 53.31 $\pm$ 0.16 \\
\midrule
\makecell[l]{Maj Train,\\ Standard Test} & 89.70 $\pm$ 0.39  & 47.92 $\pm$ 1.78 \\
\midrule
\makecell[l]{Maj Train,\\ Mean in Test} & \textbf{90.49 $\pm$ 0.12} & 56.49 $\pm$ 0.57 \\
\midrule
\makecell[l]{Maj Train, \\Maj Test (CFNN)} & 90.20 $\pm$ 0.14 & \textbf{60.16 $\pm$ 0.67} \\
\bottomrule
\end{tabular}
\end{sc}
\end{small}
\end{center}
\vskip -0.1in
\end{table}

\subsection{Sampling Generalization}

During training, we always sample $n_{train}$ times from the CFNN model.
This introduces a new notion of generalization, \emph{sampling generalization}, which is a CFNN's ability to improve its empirical random accuracy when supplied with more samples than it saw during training.
It is not immediately obvious that performance should improve.

To test if our CFNNs are able to generalize to higher sample counts, we ran a series of experiments on the models of \ref{expr:hypernet} and \ref{expr:dropout}. We trained each network with varying number of samples $n_{train}$, and measured their empirical random accuracy with several sample counts $n_{test}$.
Results for ResNet20 with dropout on CIFAR100 are presented in Figure \ref{fig:acc_ntest_ntrain}. For more results see Appendix \ref{appendix:sampling_generalization}.
The network did learn to generalize to higher sample counts.
This illustrates that CFNNs can learn to exhibit qualities of randomized algorithms which were our original inspiration.

\begin{figure}[t]
\begin{center}
\centerline{\includegraphics[width=\linewidth]{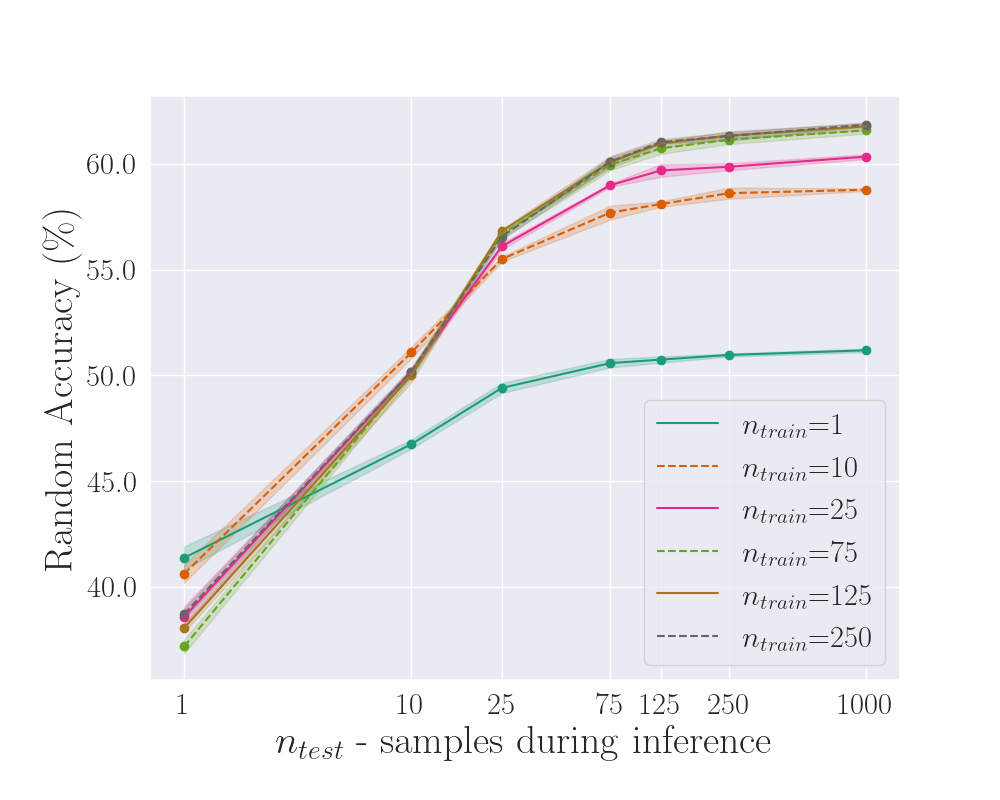}}
\caption[asd]{
    \small Empirical random accuracy of a ResNet20 with Dropout, plotted against the number samples taken during inference $n_{test}$. Accuracy improves as more samples are taken, even though the model only sees $n_{train}$ samples during training. This demonstrates sampling generalization. Shaded areas indicate confidence bounds of 1 standard deviation from the expected value. 
}
\label{fig:acc_ntest_ntrain}
\end{center}
\vskip -0.4in
\end{figure}

\section{Conclusion}
In this work, we've shown strong evidence to support the claim that NNs with access to random can achieve improved accuracy when using amplification.
We've shown that almost any classification problem can be solved with only linear classifiers, as long as they're sampled from a complex enough distribution. 
We've given an example to such a problem where a CFNN offers exponential improvement over deterministic networks.
Finally we gave strong empirical evidence to the efficacy of CFNNs, which achieved significant improvements over deterministic baselines.
However there are still several open questions. We only deal with majority amplification, and others might achieve even better performance. We also suspect that training CFNNs could be done more effectively than by using SGD. 
Finally, using CFNNs in other settings such as regression or generative modelling could produce performance improvements as well.

\nocite{*}

\bibliography{bibliography}
\bibliographystyle{icml2022}

\newpage
\appendix
\onecolumn
\section{Another Solution to the $d$-Dimensional Ball Problem}
\label{appendix:tangents_ball}
Here we present a more geometric solution than the one described in Section \ref{thm:nd_ball}. The idea is to sample lines tangent to the ball, and again add a correcting coin toss. We only give the 2-d case, which can be extended to any dimension.

We use two random values. The first is a uniform angle, $\theta \sim U[0, 2\pi]$. Using this angle we construct a linear classifier tangent to a circle of radius $b$, at angle $\theta$. That is, $h(x;\theta) = \mathrm{sgn}\left((cos\theta,sin\theta)\cdot x - b\right)$. With this classifier, every point inside a ball of radius $b$, $B_b$, is always  classified as "$-1$". 
Points outside the ball $B_b$ are classified correctly if the generated classifier $h_\theta$ "faces towards" the point. The probability that this happens is proportional to the arc length between two tangent lines drawn from the point to the circle of radius $R$, as shown in figure \ref{fig:tangents_ball}. This arc has length $2\arccos{\frac{b}{\norm{x}_2}}$, which is less than $\pi$ for all $\norm{x}_2 \geq b$; i.e., all points outside the circle have probability of less than $0.5$ to be correctly classified.
We correct this by using a second variable, $t\sim Ber(\alpha)$ for some $\alpha>\frac{1}{2}$. If $t=0$, we use the constant classifier $h_+(x) = 0 \cdot x + 1$, which classifies all samples as "1". 

Now, samples outside the circle are classified correctly with probability $\Pr[h(x)=1] = (1-\alpha) + \alpha \cdot \frac{1}{\pi}\arccos{\frac{b}{\norm{x}_2}}$, which is greater than $0.5$ for $\norm{x}_2 \geq b/cos\left(\frac{\pi}{2}(1-\frac{1}{2\alpha})\right)$. Choosing $b = R\cdot cos\left(\frac{\pi}{2}(1-\frac{1}{2\alpha})\right)$, ensures that all points except the $R$-radius circle itself have probability of success greater than $0.5$, meaning we can correctly classify any sample from $\R^2$ with arbitrarily good probability by sampling enough linear classifiers and taking the majority.

However, given an input $(\theta, t)\sim U[0,2\pi]\times Ber(2/3)$, calculating the linear classifier requires calculating $sin(\theta)$ and $cos(\theta)$ which requires a deep network \cite{DBLP:journals/corr/Telgarsky16}.
As per the discussion in Section \ref{section:comput_or_sample}, we are able to replace a function which is hard to compute with a function which is easy to sample from, $\mathcal{N}(0,I_d)$. This explains the power of the classifier $\hat{h}$ from Section \ref{thm:nd_ball}.

\begin{figure}[hb]
\vskip 0.2in
\begin{center}
\centerline{\includegraphics[width=6cm]{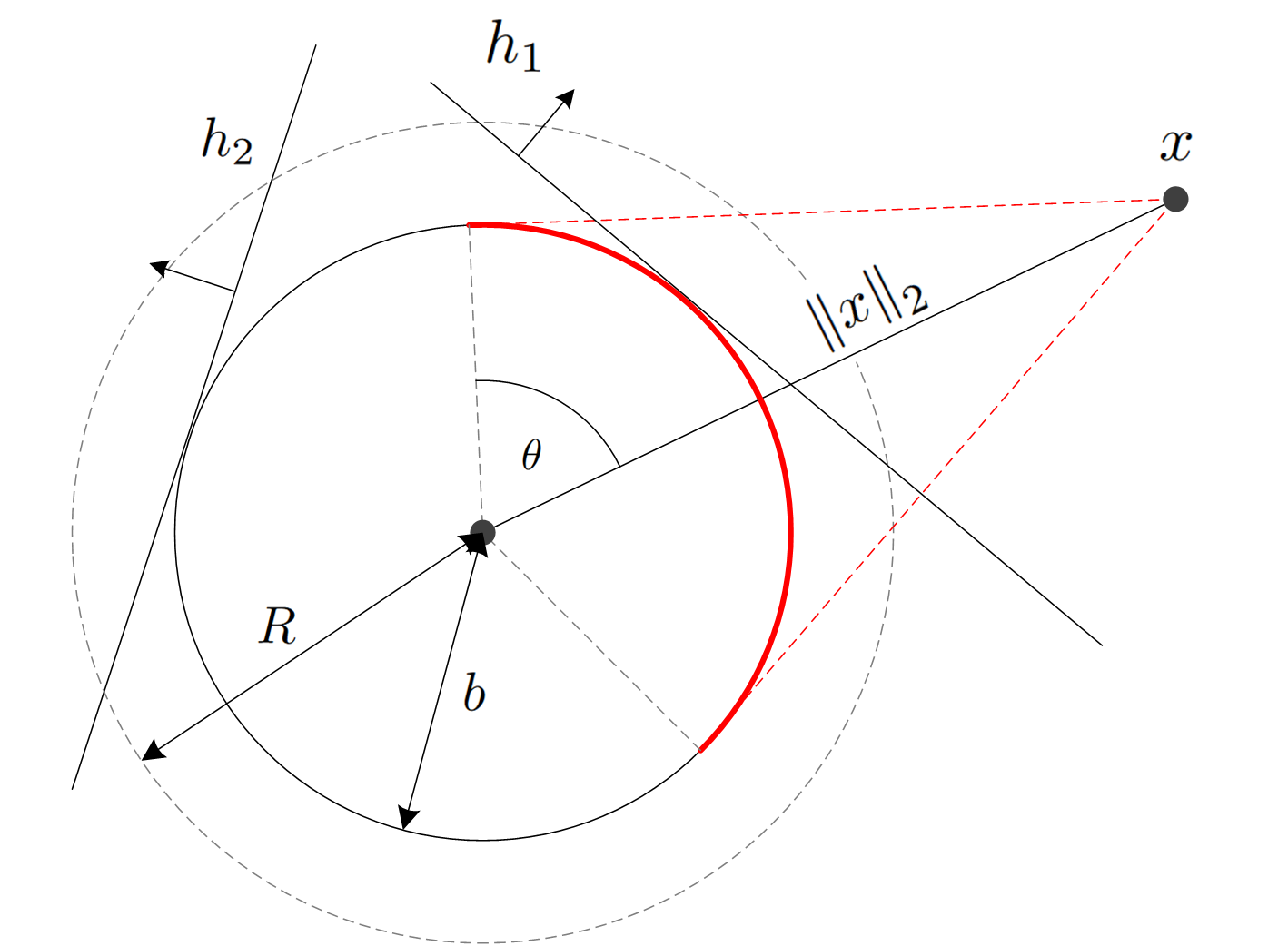}}
\caption{
    \small For some query point $x$ in $\R^2$, the tangent $h_1$ to a ball of radius $b$ forms a linear classifier. $h_1$ classifies $x$ correctly as "out of the ball".
    For the tangent classifier $h_2$ however, the point $x$ is on the wrong half-plane, so $h_2$ classifies $x$ as "inside the ball".
    All tangents $h$ which classify $x$ correctly are sampled from the arc spanned between the two tangent lines that pass through $x$, marked in red.
    The length of this arc is $2\arccos\left(\frac{b}{\norm{x}_2}\right)$. By adding a coin $t\sim Ber(\alpha)$, we can choose $b=R\cdot cos\left(\frac{\pi}{2}(1-\frac{1}{2\alpha})\right)$ such that the probability of success becomes $1/2$ exactly on the surface of a ball of radius $R$, and is greater than $1/2$ everywhere else.
}
\label{fig:tangents_ball}
\end{center}
\vskip -0.2in
\end{figure}

\section{Comparison of $\hat{h}$ with deterministic networks}
\label{appendix:comp_safran}

Comparison with \cite{DBLP:journals/corr/SafranS16} is a bit nuanced.
Their work measured accuracy as a regression problem; given some input distribution $\mu$ on $\R^d$, they define the error as MSE: $\int_{\R^d}\left(f_R(x)-\hat{h}(x)\right)^2\mu(x)dx = \norm{f_R-\hat{h}}_{L_2(\mu)}^2$.
They show that for $f_R=\sign\left(\norm{x}_2-R\right)$, and any 2-layer network $g$ approximating it \footnote{With activations satisfying some mild assumptions. Importantly, ReLU, sigmoid and threshold activations are allowed.}, there exists an input distribution $\nu:\R^d\to\R_+$ such that $\norm{f_R-g}^2_{L_2(\nu)}>\frac{c}{d^4}$ for some constant $c$, unless $g$ has width $\Omega(e^d)$.
In order to compare $g$ with $\hat{h}$ we define $\hat{h}_n(x)=\mathrm{maj}\{\hat{h}(x;u_i,t_i)\}_{i=1}^n$, i.e. we simply use the numerical value of the output of the majority of $n$ samples from $\hat{h}(x)$ as the output of $\hat{h}_n$.
We can now measure $\norm{f_R-\hat{h}_n}^2_{L_2(\mu)}$, but its value is random since $\hat{h}_n$ is random.

The simplest solution is to take the expectation of this error term and require it to be less than some $\epsilon$. Note that since $x$ and $\hat{h}$ are independent variables, this expectation is the same as the $L_2(\mu\times\gamma)$ norm of $f_R-\hat{h}_n$, where $\gamma$ represents the distribution of $\hat{h}_n$: $\E_{\hat{h}_n}\left[\norm{f_R-\hat{h}_n}^2_{L_2(\mu)}\right]=\norm{f_R-\hat{h}_n}^2_{L_2(\mu\times\gamma)}$.

\subsection{Approximability of $f_R$ with $\hat{h}_n$ in MSE}
We begine with a proof of Theorem \ref{thm:any_mu_exists_n}.

\begin{proof}
    Let $R_z=\{x\in\R^d\mid\norm{x}_2\in[R-z,R+z]\}$ for any $z>0$, and define $\phi = \sup\{z\mid\int_{R_z}d\mu\leq\epsilon/2\}$.
    In other words, we look at the thickest possible spherical shell around the surface of the sphere of radius $R$, such that the mass given to it by $\mu$ is at most $\epsilon/2$. As long as $\mu$ doesn't assign more than $\epsilon/2$ mass to the surface of the $R$ sphere itself, $\phi$ is well defined. To avoid complications, we simply disallow distributions $\mu$ with $\int_{\mathbb{S}^{d-1}(R)}d\mu>0$.
    We denote said shell as $R_\phi$.
    Next, define $\epsilon_p(\phi)=\min_{x\in\R^d\setminus R_\phi}\left\{\abs{\mathbf{p}_1^h(x)-\frac{1}{2}}\right\}$ where $\mathbf{p}^h_1(x)=\Pr\left[h(x)=1\mid x\right]=1-\alpha\Phi(\frac{b}{\norm{x}_2})$.
    Since $\mathbf{p}_1^h(x)$ is a strictly monotonicaly increasing function of $\norm{x}$, the value of $\epsilon_p(\phi)$ is $\min\left\{\frac{1}{2} - \alpha\Phi(\frac{b}{R+\phi}),\alpha\Phi(\frac{b}{R-\phi})-\frac{1}{2}\right\}$.
    
    Before we continue the proof, we state a useful lemma:
    \begin{lemma} \label{lemma:max_prob_beta}
        $\forall\beta\in(0,1]:\exists n\in\mathbb{N}$ such that $\max_{x\in \R^d\setminus R_\phi}\left\{\Pr[\hat{h}_n(x)\neq f_R(x)\mid x]\right\} \leq \beta$.
    \end{lemma}
    \begin{proof}
        First note that $\max_{x\in \R^d\setminus R_\phi}\left\{\Pr[\hat{h}_n(x)\neq f_R(x)\mid x]\right\}$ is achieved when $\Pr[\hat{h}(x)=f(x)]$ is minimized; i.e., for the same points $x$ that determine the value of $\epsilon_p(\phi)$.
        Assuming w.l.o.g that $\norm{x}_2>R$ for these points, i.e. $f_R(x)=1$, and using $[\cdot]$ Iverson's Bracket, the event $\left[\hat{h}_n(x)\neq f(x)\right]$ is the same event as $\left[\mathbf{p}^h_1(x)-\frac{1}{n}\sum_{i=1}^n\hat{h}_i(x)>\epsilon_p(\phi)\right]=\left[\frac{1}{n}\sum_{i=1}^n\hat{h}_i(x)<\mathbf{p}^h_1(x)-\epsilon_p(\phi)\right]=\left[\frac{1}{n}\sum_{i=1}^n\hat{h}_i(x)<1/2\right]$, since both represent the event of getting $\hat{h}(x)=0$ more than $n/2$ times, which would result in an error after majority.
        Given $\beta\in(0,1]$ we now apply Hoeffding's inequality to find $n$ such that $\Pr[\mathbf{p}^h_1(x)-\frac{1}{n}\sum_{i=1}^n\hat{h}_i(x)>\epsilon_p(\phi)]<\beta$, i.e. $n\geq\frac{\ln(1/\beta)}{2\epsilon_p(\phi)^2}$.
    \end{proof}
    
    Since the images of $f_R$ and $\hat{h}_n$ are in $\{-1,1\}$, we have $(f_R-\hat{h}_n)^2 = \left[f_R(x)\neq\hat{h}_n(x)\right]$.
    We can now calculate the error:
    \begin{align*}
        \norm{f_R-\hat{h}_n}^2_{L_2(\mu\times\gamma)} 
        &=\int\left[f_R(x)\neq\hat{h}_n(x)\right]d\mu(x) d\gamma(\hat{h}_n) 
        = \Pr\left[f_R(x)\neq\hat{h}_n(x)\right] \\
        &= \int_{\R^d}\Pr\left[f_R(x)\neq\hat{h}_n(x)\mid x \right]d\mu(x) \\
        &= \int_{R_\phi}\Pr\left[f_R(x)\neq\hat{h}_n(x)\mid x \right]d\mu(x) + \int_{\R^d\setminus R_\phi}\Pr\left[f_R(x)\neq\hat{h}_n(x)\mid x \right]d\mu(x) \\
        &\leq \int_{R_\phi}1\cdot d\mu(x) + \int_{\R^d\setminus R_\phi}\Pr\left[f_R(x)\neq\hat{h}_n(x)\mid x \right]d\mu(x) \\
        &\leq \frac{\epsilon}{2} + \max_{x\in \R^d\setminus R_\phi}\left\{\Pr[\hat{h}_n(x)\neq f_R(x)\mid x]\right\}\int_{\R^d\setminus R_\phi}d\mu(x) \\
        &\leq \frac{\epsilon}{2}+\beta(1-\frac{\epsilon}{2})
    \end{align*}
    
    Setting $\beta=\frac{\epsilon/2}{1-\epsilon/2}$, we get $\norm{f_R-\hat{h}_n}^2_{L_2(\mu\times\gamma)} \leq\epsilon$ as long as $n\geq\frac{\ln(1-\epsilon/2) - \ln(\epsilon/2)}{2\epsilon_p(\phi)^2}$. This concludes the proof of Theorem \ref{thm:any_mu_exists_n}.
\end{proof}

\subsection{Comparison of Computational Complexity}
We've shown that $\hat{h}_n$ can approximate $f_R$ to arbitrary accuracy with $\mathcal{O}(1)$ space complexity.
Using Theorem \ref{thm:any_mu_exists_n}, now we can compare the computational complexity of $\hat{h}_n$ with the smallest deterministic 2-layer network which approximates $f_R$. Denote this deterministic network with $g(x)$.
First, we define computational complexity of a network as the number of times a neuron is used. For deterministic networks connected with no cycles, it is simply the number of neurons.
For CFNNs, it is the number of neurons times $n$, the number of times they are sampled from.

Since $\hat{h}$ has exactly 2 neurons, and we use it $\Omega(n)$ times, its computational complexity is $\Omega(n)=\Omega(\frac{\ln(1-\epsilon/2) - \ln(\epsilon/2)}{2\epsilon_p(\phi)^2})$.
The direct dependence on $\epsilon$ is logarithmic.
There is also an indirect dependence on $\mu$ through $\epsilon_p(\phi)$.

Assume $\mu$ assigns a mass of $\epsilon/2$ to a very thin spherical shell $R_\phi$, such that $\phi\propto e^{-d/2}$. Now, since $\mathbf{p}_1^h(x)$ is differentiable, we can approximate it in $\R_\phi$ using a linear function of $\norm{x}_2$, as $\phi$ is very small. Therefore $\epsilon_p(\phi)\propto\phi$, and we have $n\propto e^d$.
This implies that by using the $n$ achieved in the proof of Theorem \ref{thm:any_mu_exists_n}, and for any $\epsilon>0$, there exists an input distribution $\mu'$ such that in order for $\norm{f_R-\hat{h}_n}^2_{L_2(\mu\times\gamma)}$ to be less than $\epsilon$, the computational complexity of $\hat{h}_n$ must be $\Omega(e^d)$. This is true for \emph{any} $\epsilon$, including $\epsilon=d^{-4}$.

However, we are not done. The distribution $\mu'$ is extremely pathological, severely more concentrated than the distribution $\nu$ used in the proof of \cite{DBLP:journals/corr/SafranS16}.
In fact, $\nu$ was designed so that its Fourier transform had sufficiently large mass in high frequency regions of its spectrum. Due to properties of the Fourier transform, this translates to relatively small regions of space with high concentration of mass. Since distributions $\mu'$ (on which $\hat{h}_n$ requires exponential complexity) are even more concentrated than $\nu$, they too will have the same properties as $\nu$, that are required to show that no sub-exponential deterministic network exists that can approximate $f_R$ with $\mu'$ as the input distributions.
In other words, when $\hat{h}_n$ is exponential, so are all deterministic neural networks.

To finish our comparison, we explicitly show that on $\nu$, our network $\hat{h}_n$ performs better.\footnote{The distribution used in \cite{DBLP:journals/corr/SafranS16} was denoted $\mu$, so we changed it to $\nu$ to avoid confusion.}

First, we must re-implement $\hat{h}$ in order to reduce its depth. Remember that $\hat{h}(x)=\sign(t(ux-b-1)+1)$. This can be implemented using two random variables, $u\sim\mathcal{N}(0,I_d)$ and $t\sim Ber(\alpha)$, and two neurons as described in Section \ref{thm:nd_ball}. We can also use two other variables, $v=ut$ and $c=t(b+1)-1$, and get $\hat{h}(x)=\sign(vx-c)$ which is a single neuron. As the definition of CFNN does not limit the distribution of the random input, this holds.

The distribution $\nu$ was chosen such that any 2-layer deterministic network $g$ would be unable to approximate the function $\Tilde{g}(x)=\sum_{j=1}^N\frac{\epsilon_j}{2}(f_{a_j}(x)-f_{b_j}(x))$, where $f_{a_j},f_{b_j}$ are defined similarly to $f_R$,  $N\in \mathcal{O}(poly(d))$, and $\epsilon_j\in\{-1,1\}$.
In other words, $\norm{\tilde{g}-g}_{L_2(\nu)}>2N\sqrt{\epsilon}$ for any 2-layer network $g$.
We can approximate each $f_{a_j},f_{b_j}$ using $2N$ instances of our network $\hat{h}^{a_j},\hat{h}^{b_j}$ with parameter distributions appropriately selected, and combine their results in the second layer as $\sum_{j=1}^N\epsilon_j\left(\hat{h}^{a_j}_n(x)-\hat{h}^{b_j}_n(x)\right)$.
Note that we perform majority amplification on each neuron before the summation, so we can use Theorem \ref{thm:any_mu_exists_n}, i.e. $\exists n\in\mathbb{N}$ s.t. $\norm{f_R-\hat{h}_n}_{L_2(\nu\times\gamma)} <\sqrt{\epsilon}$. Also, this new network is a 2-layer network with $N$ neurons, with computational complexity $\Omega(nN)$.

We can now use the triangle inequality:
\begin{align*}
    \norm{\tilde{g}-\sum_{j=1}^N\epsilon_j\left(\hat{h}^{a_j}_n(x)-\hat{h}^{b_j}_n(x)\right)}_{L_2(\nu\times\gamma)} 
    \leq \sum_{j=1}^N\abs{\epsilon_j}\left(\norm{f_{a_j}-\hat{h}^{a_j}_n}_{L_2(\nu\times\gamma)}+\norm{f_{b_j}-\hat{h}^{b_j}_n}_{L_2(\nu\times\gamma)}\right)
    \leq 2N\sqrt{\epsilon}
\end{align*}

We now only need to show that the number of samples $n$ required is polynomial in $d$. This requires examining the distribution $\nu$ as described in \cite{DBLP:journals/corr/EldanS15} on which the work in \cite{DBLP:journals/corr/SafranS16} is based.
The important notes about $\nu$ are: it is only a function of $\norm{x}_2$; it is nonzero on at most $\mathcal{O}(N)$ intervals $\Delta_j=[a_j,b_j]$; and it is constant up to a polynomial factor, i.e. $\sup_{x\in\Delta_j}\{\nu(x)\}\leq\inf_{x\in\Delta_j}\{\nu(x)\}\cdot(1+poly(d^{-1}))$ for some polynomial $poly(d^{-1})$.
Another important note is that $b_j-a_j$ are in $\mathcal{O}(1/N)$.
As was shown in the discussion above, $n$ will be exponential only if the distribution $\nu$ were to assign a mass of $\epsilon/2$ to a very thin spherical shell $[a_j,a_j+\phi]$ such that $\phi\propto e^{-d/2}$.
However, this is not the case. Assume by contradiction that $\exists m\leq N: \int_{[a_m,a_m+\phi]}d\nu(x)=\epsilon/2$ and $\phi\propto e^{-d}$:
\begin{align*}
    &\int_{[a_m,a_m+\phi]}d\nu(x)=\frac{\epsilon}{2}\leq\sup_{x\in\Delta_m}\{\nu(x)\}\phi\leq\inf_{x\in\Delta_m}\{\nu(x)\}\left(1+poly(d^{-1})\right)\phi \\
    &\Rightarrow \inf_{x\in\Delta_m}\{\nu(x)\}\geq\frac{\epsilon}{2\left(1+poly(d^{-1})\right)\phi}\\
    &\Rightarrow \int_{\R^d}d\nu(x)=\sum_{j=1}^N\int_{\Delta_j}d\nu(x) \geq \sum_{j=1}^N\left(b_j-a_j\right)\inf_{x\in\Delta_j}\{\nu(x)\}\geq \left(b_m-a_m\right)\inf_{x\in\Delta_m}\{\nu(x)\} \\
    &\Rightarrow \int_{\R^d}d\nu(x) \geq \frac{\left(b_m-a_m\right)\epsilon}{2\left(1+poly(d^{-1})\right)\phi}
\end{align*}

Since $\phi\propto e^{-d/2}$, and $\left(b_m-a_m\right)\in\mathcal{O}(1/poly(d))$, then for sufficiently large $d$ we get $\frac{\left(b_m-a_m\right)\epsilon}{2\left(1+poly(d^{-1})\right)\phi}\propto e^d\cdot\epsilon/poly(d) > 1$, i.e. $\int_{\R^d}d\nu(x)>1$ and $\nu$ is not a distribution.
Hence no such $m$ exists, i.e. $\phi^{-1}$ is at most polynomial in $d$. The asymptotic behavior of $\phi$ is thus $\Omega(\epsilon/poly(d))$, and since $\epsilon_p(\phi)\propto\phi$ we have that $n$ is $\Omega(poly(d)\cdot\epsilon^{-2}\ln(1/\epsilon))$.

To conclude, we've shown that for the distribution $\nu$ and function $\tilde{g}$, even though no deterministic 2-layer network can approximate $\tilde{g}$ to a better accuracy than $\epsilon=d^{-4}$ unless it has an exponential width, the network $\hat{h}$ can be used to approximate $\tilde{g}$ to any accuracy $\epsilon$ such that $\epsilon\in\Omega(poly(d^{-1}))$, with only 2 layers, $N=poly(d)$ neurons, and computational complexity $\Omega(poly(d))$.

\section{Proof of Theorem \ref{thm:strong_gens}}
\label{appendix:full_universality}

In Section \ref{section:theorem} we've shown the first step of the proof, for the 1-d case. First we extend Theorem \ref{thm:sep-1d} to $d$ dimensions.

\begin{definition} (Separates in Probability)
    Given a function $f:\R^d\to\R$ and a distribution $\gamma$ of classifiers $h$, we say that $h\sim\gamma$ \textbf{\emph{separates $f$ in probability}} if $\forall x\in\R^d,y\in\R: \Pr_{h\sim\gamma}\left[h(x_1, ..., x_d ,y)=\sgn(f(x)-y)\mid x, y\right] > 0.5$ and $\Pr_{h\sim\gamma}\left[h(x_1, ..., x_d, f(x))=1\mid x\right] = 0.5$.
\end{definition}
$x_i$ is the $i$-th coordinate of $x$.

\begin{theorem}\label{thm:sep-nd}
Let $f:\R^n\to\R$ be a $K-Lipschitz$ function. Then there is a distribution $\gamma$ of $L_1$-cones that separates $f$ in probability.
\end{theorem}
An $L_1$-cone with origin $p$ is $C_K^1(p)=\left\{x\in\R^{d+1}\mid K\cdot\norm{(x_1, ..., x_d) - (p_1, ..., p_d)}_1 \leq \abs{x_{d+1}-p_{d+1}}\right\}$.

\begin{proof}
    The proof is very similar to that of Theorem \ref{thm:sep-1d}, except we replace $C_K(p)$ with an $L_1$ cone of slope $K$, $C_K^1(p)$.
    This cone shares the properties presented in Lemma \ref{lemma:kcone}, hence the proof is immediately applicable.
\end{proof}

\label{nl1cone}
\paragraph{$L_1$-cone Network} The base network $N$ must now be an implementation of this cone. That can be done in 2 layers and ReLU activations; the first layer computes $o_i^+=\max(0, Kx_i-Kp_i)$ and $o_i^-\max(0,Kp_i-Kx_i)$, i.e. two neurons for every input dimension.
Their sum is $o_i^++o_i^-=\abs{Kx_i-Kp_i}$ which is then used in the second layer to compute $\sign\left((o_{d+1}^++o_{d+1}^-)-\sum_{i=1}^d (o_i^++o_i^-)\right)=\sign\left(\abs{x_{d+1}-p_{d+1}}-\norm{x_{1...d}-p_{1...d}}_1\right)$. Note that $K$ and $Kp_i$ are the weights and the biases of the first layer. $N$ has $2d+1$ neurons.

We use $L_1$-cones since we were unable to find a distribution of hyperplanes (linear classifiers in $\R^{d+1}$) which shared the properties of Lemma \ref{lemma:kcone}. The naive idea is best explained with $d=3$: given a point on the surface of the function, $p=(x_0, x_1, f(x))$, we sample a direction $v=(Kcos(\theta), Ksin(\theta), 1)$ and use the hyperplane $h^p_\theta(x)=v\cdot x - v\cdot p$. This is the same as sampling the generating lines of an $L_2$-cone whose origin is $p$, $C_L^2(p)$. However, this cone does not have the properties described in Lemma \ref{lemma:kcone}.

To see why, consider a point $q$ with some $q_3>p_3$, outside of $C_L^2(p)$. We observer the hyperplane $\left\{(x, y, q_3)\mid x, y \in \R\right\}$, i.e. the surface of height $q_3$ along the 3rd axis and parallel to the first two axes.
This surface contains $q$, and intersects $C_K^2(p)$ on a disk with some radius $r(\abs{q_3-p_3})$.
Determining if $q\in C_K^2(p)$ then reduces to determining if it is inside this disk.
As we've already seen in Section \ref{thm:nd_ball}, the distribution of linear classifiers that classifies the ball has a probability function which is not constant inside or outside of the ball. Hence Lemma \ref{lemma:kcone} does not apply.

In Theorem \ref{thm:sep-1d} we've used a cone $C_K(p)$ defined by two linear classifiers, $h^{p+}, h^{p-}$. This cone is also an $L_1$-cone, so we can continue with the full proof of Theorem \ref{thm:strong_gens} for arbitrary $d$.

\begin{theorem}\label{thm:cls-nd}
Let $f:\R^d\to\{-1,1\}$ be a function such that there exists a K-Lipschitz function $f':\R^d\to\R$ such that $\forall x\in\R^d:\sgn(f'(x))=f(x)$.
Then there is a distribution $\gamma$ of $L_1$-cones that classifies $f$ in probability.
\end{theorem}
\begin{proof}
Let $\gamma'$ be a distribution of $L_1$ cones that separates $f'$ in probability, which exists by Theorem \ref{thm:sep-nd}.
Define $\gamma$ by sampling $h'\sim\gamma'$ and calculating $h(x)=h'(x_1, ..., x_d, 0)$. Note that $h(x)$ is a linear separator in $\R^d$, hence $\gamma$ is well defined.
Verifying that $\gamma$ classifies $f$ in probability, we get:
\begin{align*}
    &\Pr_{h\sim\gamma}\left[h(x)=f(x)\mid x\right] = \int_h\left[h(x)=f(x) \mid x\right]d\gamma \\
    &= \int_{h'}\left[h'(x,0)=f(x) \mid x\right]d\gamma' = \int_{h'}\left[h'(x,0)=\sgn(f'(x)) \mid x\right]d\gamma'\\ 
    &= \Pr_{h'\sim\gamma'}\left[h'(x,0)=\sgn(f'(x)) \mid x\right] > \frac{1}{2}
\end{align*}
We've used the fact that $f(x)=\sgn(f'(x))$, and the last inequality is true since $\gamma'$ separates $f$ in probability.
$[\cdot]$ is Iverson's Bracket.
\end{proof}

We can now complete our proof:
\begin{proof} [Proof of Theorem \ref{thm:strong_gens}]
Since $f'$ is continuous, the sampling procedure defined in \ref{thm:sep-1d} is a continuous function $T:\R^d\to\R^{d+1}$ with $T(t)=Kp=K\left(t_1, t_2, ..., t_d, f(t)\right)$. The function $T$ computes the parameters of a \hyperref[nl1cone]{network $N$ which computes $C_K^1(p)$.}
$K$ is the Lipschitz constant of the function $f'$.
Given a distribution $\zeta$ over $\R^d$ with $\supp\{f\}\subseteq\supp\{\zeta\}$, the pushforward measure $\gamma'=\zeta\circ T^{-1}$ is a distribution of $L_1$-cones.

Now, given a query point $q\in\R^{d+1}$, recall the definition of $P_b(q)$ as the set of points $p\in\R^{d+1}$ on the surface of $f(x_1, ..., x_d)$ such that $q\in C_K^1(p)$.
Let $A(q)=T^{-1}(P_b(q))=\left\{(p_1,...,p_d)\mid p\in P_b(q)\right\}\subseteq\R^d$. By construction, $A(q)\subseteq\supp\{f\}$, hence $A(q)\subseteq\supp\{\zeta\}$ from the requirement on $\zeta$.

By change of variables, we have $\int_{P_b(q)}d\gamma'(p) = \int_{A(q)}d\zeta(t) > 0$. Since this is the only requirement on $\gamma'$ to separate $f'$ in probability, we immediately get a distribution $\gamma$ of $L_1$ cones which classifies $f$ in probability, by Theorem \ref{thm:cls-nd}.

Since $T$ is a continuous function, and from universality of neural networks \cite{universality_in_depth}, there exists a neural network $G$ which approximates $T$ to arbitrary precision.
Hence, for $r\sim\zeta$, and using a base network $N$ as defined above, we have $N(x;G(r))\sim\gamma$. In other words, $N(x;G(r))$ classifies $f$ in probability.
\end{proof}

\subsection{Constraints on $f$} \label{appendix:constraint_on_f}
The constraint the $f$ have an associated K-Lipschitz function $f'$ with the same sign is ill-posed. Let $\mathcal{X}\subset\R$ be the set of points which are the boundaries of regions with constant value of $f$. On these points $f'(x)=0$, since its sign is different on either side of said boundaries and it must be continuous.
This means that we cannot use the fact that $\gamma$ separates $f'$ to calculate $f$ on $\mathcal{X}$, since $\Pr[h(x,0)=1\mid x\in\mathcal{X}]=0.5$, and no amplification is possible.
We can either require that $f(x)=0$ on $\mathcal{X}$, or exclude these points from consideration by demanding that for any input distribution $\mu$ we have $\int_{\mathcal{X}}d\mu(x)=0$.

Now that $f'$ is well defined, an implication of it being K-Lipschits is that $f$ cannot be 'too-oscillatory'. The constant $K$ can be thought of as representing the frequency of oscillations of $f$, since $f'$ can be found as a Fourier series. The fact $f'$ exists then implies that the frequency spectrum of $f$ has most of its mass below some frequency, otherwise such a series could not use a finite number of summands, whereby $f'$ would not be Lipschitz.
\section{Multi-class CFNNs and Randomness}
\label{appendix:delta}

We hypothesize that as a classification problem has more classes $C$, CFNNs become easier to train than their deterministic equivalents. For example, a CFNN that on any input $x$ outputs all classes with almost-equal probability $\frac1C-\epsilon$, except a slightly higher probability for the correct class $\frac1C+(C-1)\epsilon$, has a perfect Random Accuracy. 
A CFNN can learn a probability distribution with only a slight bias to output the correct classification, which with more classes could be easier to accomplish; there is more "room for error".

to quantify this we define $\delta_N(x) = \mathbf{p}_{f(x)}(x) - \max_{j\neq f(x)}(\mathbf{p}_j(x))$, which effectively measures "how random" is the output of the network. The function $\delta_N$ evaluates how much probability mass the a network has invested in classes other than the predicted class. If $\delta_N(x)>0$, then that input will be considered as correctly classified in random accuracy terms.

Denote $\hat{\delta}_N(x)$ as the empirical estimate of $\delta_N(x)$.
Now we have $\widehat{RA}=\frac{1}{\abs{S}}\sum_{x\in S}[\hat{\delta}_N(x)>0]$, where $[\cdot]$ is Iverson's bracket. By plotting the values of $\hat{\delta}_N$ for a given dataset in decreasing order, we can observe the empirical random accuracy as the value of the $x$-axis intercept (Figure \ref{fig:delta_graph}).

As a CFNN behaves more randomly, this graph would be more horizontal; when the CFNN has greater accuracy, its $x$ intercept would move to the right.
This visualization is a great tool to analyze the performance of CFNNs, which is useful for the following ablations in Appendix \ref{appendix:hypernet_ablations}.

\begin{figure}[t]
\vskip 0.2in
\begin{center}
\centerline{\includegraphics[width=8cm]{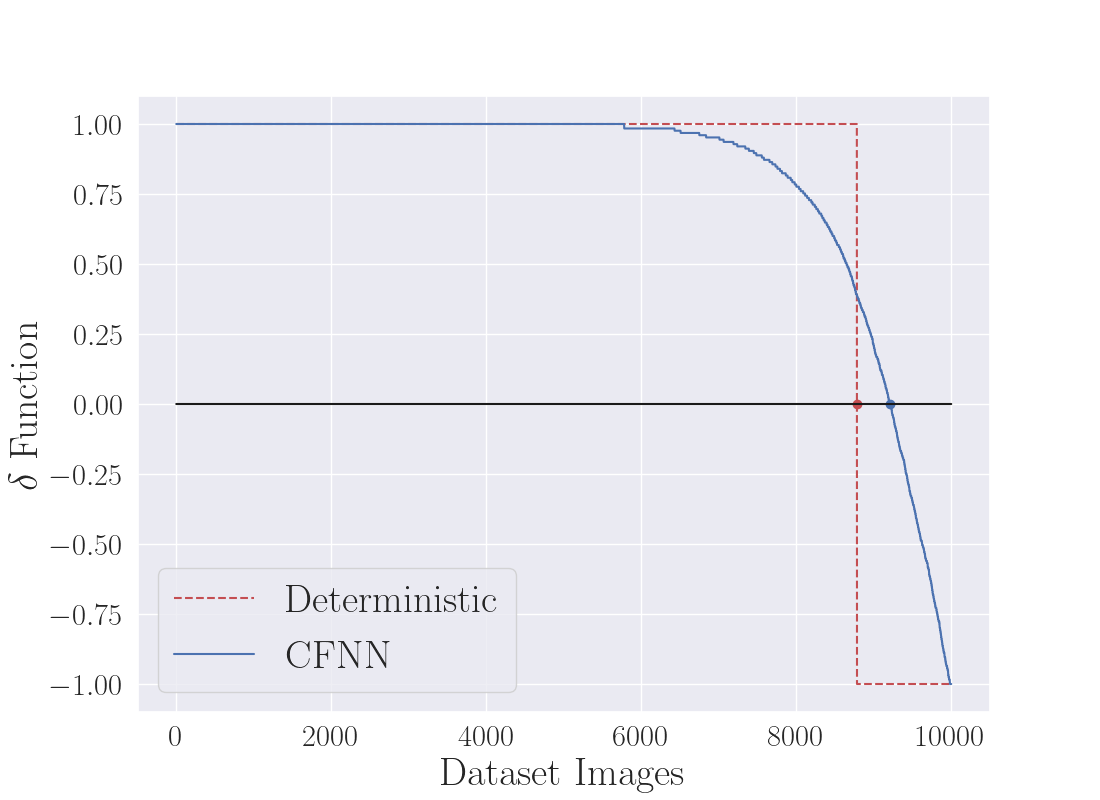}}
\caption[asd]{
    \small Graph of $\hat{\delta}_N(x)$ of a deterministic and a CFNN network. The $x$-axis are images from CIFAR10 test set.
    Since a deterministic network's $\delta$ function is in $\{1, -1\}$, its graph is a step function.
    The $\delta$ of a CFNN network however can take arbitrary values in $[-1,1]$.
    The $x$-axis intercept is the networks accuracy over the dataset. For the deterministic network shown, it is 87.8\%; for the CFNN it is 91.8\%.
    This graph is a visualization of the constraint on the search space of deterministic NNs, which is removed for CFNNs.
    The curve of the graph is drawn by sorting the images such that $\delta$ is decreasing. The ordering of the images are different between the two curves.
}
\label{fig:delta_graph}
\end{center}
\vskip -0.2in
\end{figure}

\section{Experiments - Hypernetwork CFNNs}
\label{appendix:hypernet_ablations}
\label{appendix:hypernet_details}
\paragraph{Implementation Details} As illustrated in Figure \ref{fig:hypernetwork_cfnn}, our Hypernetwork CFNN implements the the two-stage sampling procedure inspired by Section \ref{section:theorem}. To be precise, our implementation performs the following,
\begin{align*}
    e(x) &= N_\text{e}(x;\theta_1)\\
    \hat e(z_1,\bar y) &= G_\text{e}(\{z_1,\bar y\};\phi_1)\\
    \theta_2(z_1,z_2,\bar y) &= G_\text{h}(\{z_2,\bar y,\hat e(z_1,\bar y)\};\phi_2)\\
    o(x,z_1,z_2,\bar y) &=N_\text{h}(e(x);\theta_2(z_1,z_2,\bar y))\\
    \Tilde{\mathbf{p}}(x) &= \mathbb{E}_{z_1,z_2\sim U([0,1]^m), \bar y \sim \text{Cat}(C, \frac1C)}
    \left[ \Tilde{I}\left(o(x,z_1,z_2,\bar y))\right)\right] \\
    \text{Loss} &= \text{CEL}(\Tilde{\mathbf{p}}(x), y).
\end{align*}

where $\Tilde{I}$ is an approximate indicator and $\{\cdot\}$ is the concatenation operator. 

For $N_\text{e}$, we use ResNet of depth 6\textbackslash10 for CIFAR10\textbackslash100 respectively, where the final linear layer outputs an embedding vector $e$ of size 256. For $G_\text{e}$, we use a one dimensional equivalent to ResNet-6, with a random input $z_1$ sampled uniformly from $[0,1]^{256}$. The categorical class input $\bar y \sim \text{Cat}(C,\frac1C)$ is represented via a one hot vector of length 10\textbackslash100 for CIFAR10\textbackslash100 respectively. The one hot vector is passed to a linear layer, yielding a class embedding of size 256, which is concatenated to the random seed $z_1$ as an additional channel. $G_\text{h}$ applies a similar one dimensional ResNet-6, that takes a random seed $z_2$ sampled uniformly from $[0,1]^{256}$, a class one hot vector $\bar y$ (processed as in $G_e$ to an embedding of size 256) and the sampled embedding $e$ as 3 channels of size 256. $G_h$ outputs $\theta_2=[w,b]$ using two linear heads. Finally, $N_\text{h}$ applies $N_\text{h}(e)=w\cdot e-b$. All ResNets have 16 channels on their first residual block, multiplied by 2 every residual block as usual.

\begin{table}[t]
\caption{\textbf{Hypernetwork CFNN Ablation Study} for the CIFAR10 and CIFAR100 benchmarks. Performance is measured in Random Accuracy (\%). The average weight experiment shows whether there is value in using a distribution of weights rather than the expected weight. The "frozen" $N_e$ experiment emphasizes that $N_e$ does not hold all the computational power. Although the the differences are small, randomness still appears to add value in some cases, especially when the problem has a large number of classes as in CIFAR100. }
\label{tbl:hypernet_ablation_results}
\vskip 0.15in
\begin{center}
\begin{small}
\begin{sc}
\begin{tabular}{l|ccccr}
\toprule
&\multicolumn{4}{c}{CIFAR10} \\
$L$ & CFNN & Average Weight & Frozen $N_e$ & Det \\
\midrule
2  & 44.05 & 44.42 & 43.87& 42.14\\
6  & 81.45 $\pm$ 0.24 & 81.86 & 81.46 & 79.71 $\pm$ 0.52 \\
10 & 89.37& 89.37 &89.14  & 88.41 \\
14 & 91.81& 91.81 & 91.90  & 91.42 \\
20  & 92.37 & 92.39 & 92.49 & 92.89   \\
\midrule
&\multicolumn{4}{c}{CIFAR100}\\
\midrule
6 & 45.29  &44.82 &43.09 & 44.23 \\
10 & 61.24$\pm$ 0.20  &61.25 &
59.10& 58.09 $\pm$0.22\\
14 & 66.48 &66.38 &65.94 &	65.16\\
20 & 69.26 &69.29 &69.02 &	67.68\\
56 & 72.00 &72.01 &71.65 &	72.23\\ 
\bottomrule
\end{tabular}
\end{sc}
\end{small}
\end{center}
\vskip -0.1in
\end{table}

\paragraph{Training Hyperparameters} Our CFNN is optimized using SGD with Cross Entropy loss, momentum=0.9 and $L_2$ weight decay of $5\times10^{-3}$, performing majority as described in Section \ref{section:training}. We train our models for 200 epochs with batch size 128. We apply a cosine learning rate scheduler with an initial learning rate 0.1. For CIFAR10, we use $n=25$ samples and for CIFAR100 we use $n=125$ samples. For comparison we train equivalent deterministic models with the exact same setting, only replacing the generated parameters $\theta_2$ with a learned linear layer. 
Little to no optimization of the hyperparameters of the training process was performed.

\begin{figure*}[t]
\vskip 0.2in
\begin{center}
    \begin{subfigure}{0.35\textwidth}
         \centerline{\includegraphics[width=7cm]{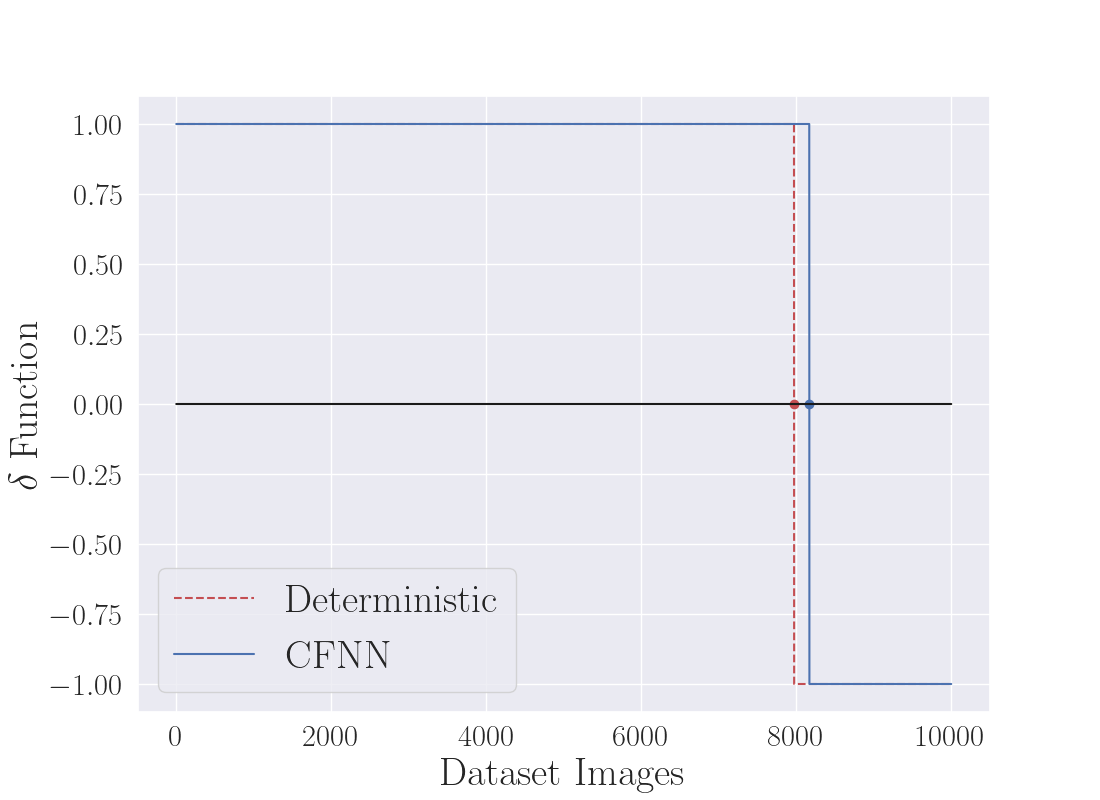}}
        \caption{}
        \label{fig:hypernet:delta:cifar10}
    \end{subfigure}
    \begin{subfigure}{0.35\textwidth}
         \centerline{\includegraphics[width=7cm]{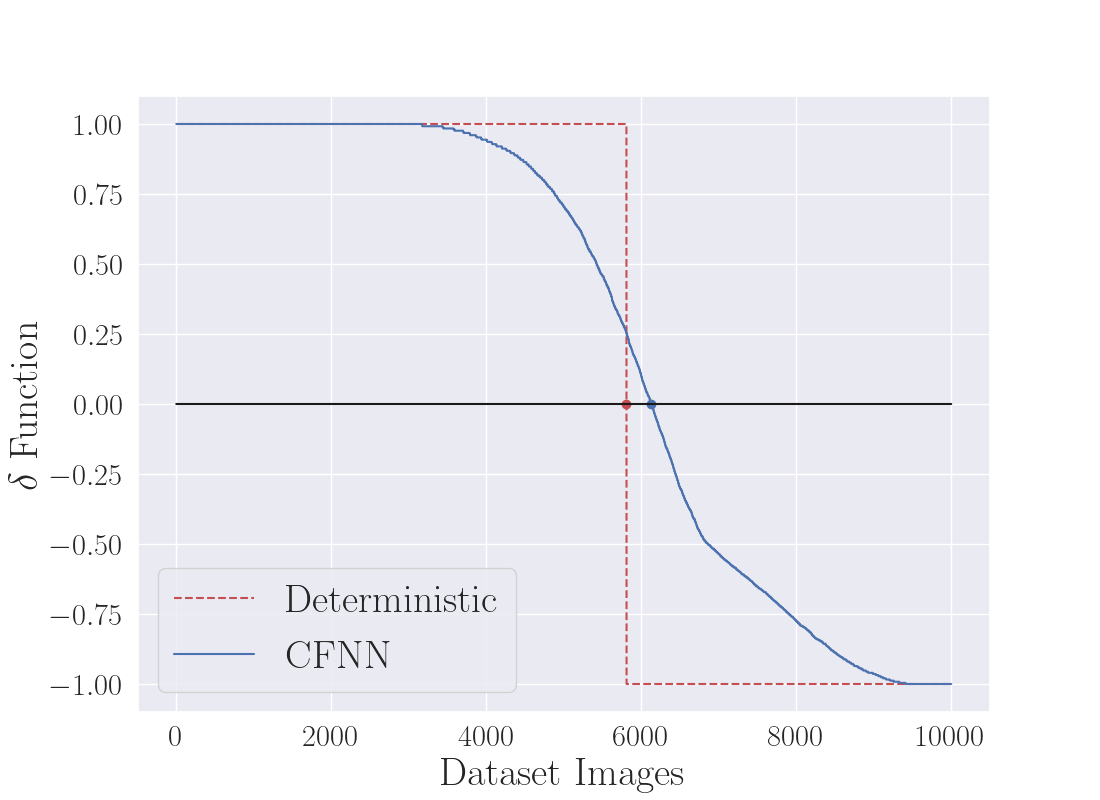}}
        \caption{}
        \label{fig:hypernet:delta:cifar100}
    \end{subfigure}
    \caption{\small (a) The $\hat\delta(x)$ function as defined in Appendix \ref{appendix:delta} for our Hypernetwork CFNN with $N_e$ depth $D=6$ on CIFAR10. The performance improvement over the deterministic equivalent is not large but statistically significant. (b) The $\hat\delta(x)$ function for our Hypernetwork with $N_e$ depth $D=10$,  CFNN on CIFAR100. It is clear that a problem with more classes benefited the CFNN.}
    \label{fig:hypernet_delta}
\end{center}
\vskip -0.2in
\end{figure*}

\paragraph{Ablation Study} We perform 2 additional experiments to ensure the perceived improvement in performance indeed results from random behavior. First, for each CFNN we compute the average weight sampled from the generator $\bar \theta_2 = \E_{z_1,z_2\sim\mathcal{Z}, \bar y\sim\text{Cat}(C,\frac1C)}\left[\theta_2(z_1,z_2,\bar y)\right]=\E\left[G_\text{h}(\{z_2,\bar y,\hat e(z_1,\bar y)\};\phi_2)\right]$, and use it instead of sampling from the generator, yielding a deterministic model.
A CFNN that has lower performance when using the average weight, has learned to gain performance purely from variance of weights. 

For our second experiment, we take the embedding network $N_e$ that was trained with the full hypernetwork model, add a standard FC linear layer in place of $N_h$ and the generators.
Training this new layer while keeping the weights of $N_e$ itself frozen,
measures if the embedding learned by $N_e$ is easily separable by a linear classifier.

Results are summarized in Table \ref{tbl:hypernet_ablation_results}.
The hypernetwork model has mostly learned a highly concentrated distribution of weights, with the only exception being for $D=6$ on CIFAR100.

These results suggest that the hypernet model learned a highly concentrated distribution, which allows replacing it with a constant parameter. This can also be seen by observing Figure \ref{fig:hypernet_delta} where on CIFAR10 the delta is very close to being a step function, as if the CFNN were entirely deterministic. On CIFAR100 the graph is more curved, but still the average of the distribution is a very good parameter for $N_h$.
However, this parameter is not found using standard SGD, as can be seen from the results using a frozen $N_e$.

When observing the results in the Dropout experiment in Section \ref{expr:dropout}, the equivalent for the average weight ablation is to replace the Dropout layer with its mean mask, i.e. a constant $p$. This is exactly what is done during standard inference with Dropout. Thus we've already measured a substantial improvement between Dropout CFNNs and their average weight equivalents.
This suggests that the method itself is not the problem, but either the hypernetwork model is a not very good candidate for CFNNs, or its training needs to be different than standard SGD.
These questions could be worthwhile to explore in future work.

\section{Dropout CFNNs}
\label{appendix:dropout_details}
We use ResNet-20 as our model, incorporating a Dropout layer with $p=0.5$ after every other activation layer and once before the last linear layer. We use the same training hyperparameters as described in Appendix \ref{appendix:hypernet_details}. The first layer has 16 channels, with the number of channels doubled on every block as usual.
We observe in Figure \ref{fig:dropout:delta} that the model has learned a highly random behaviour, with a very horizontal delta function on CIFAR100.
As described in the ablation study in Appendix \ref{appendix:hypernet_ablations}, removing the contribution of randomness is the same as training with amplification and using the expected mask of the Dropout layer during inference, which is the standard inference method with Dropout.
This was already done in Table \ref{tbl:dropout_results}.

\begin{figure*}[b]
\vskip 0.2in
\begin{center}
    \begin{subfigure}[b]{0.35\textwidth}
         \centerline{\includegraphics[width=7cm]{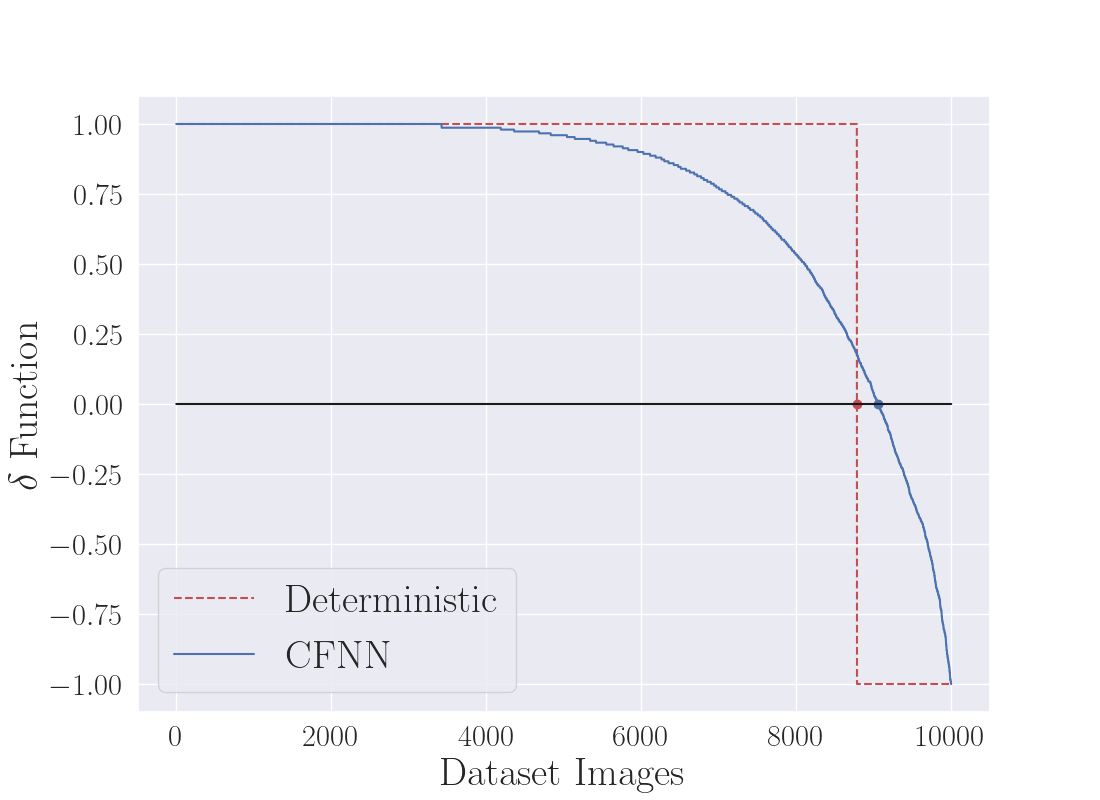}}
        \caption{}
        \label{fig:dropout:delta:cifar10}
    \end{subfigure}
    \begin{subfigure}[b]{0.35\textwidth}\centerline{\includegraphics[width=7cm]{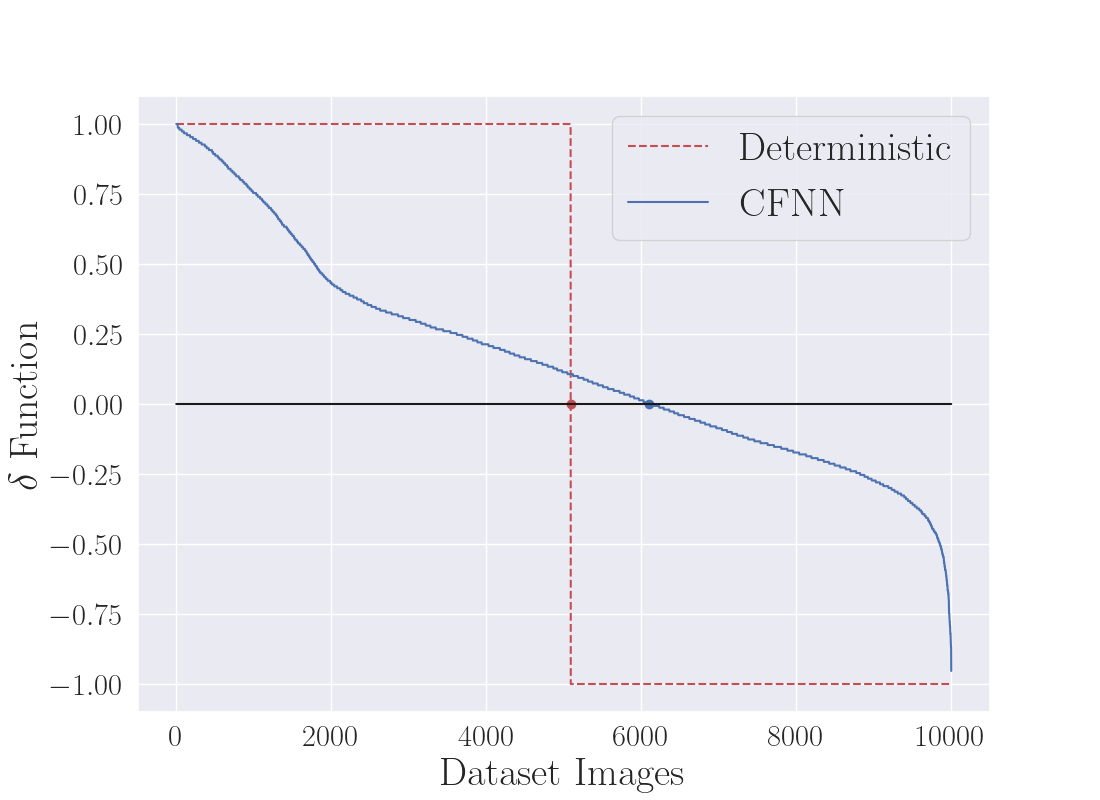}}
        \caption{}
        \label{fig:dropout:delta:cifar100}
    \end{subfigure}
    \caption{\small (a) The $\hat\delta(x)$ function as defined in Appendix \ref{appendix:delta} for our Dropout CFNN on CIFAR10, indicating randomness. The performance improvement over the deterministic equivalent is not large but statistically significant. (b) The $\hat\delta(x)$ function for our Dropout CFNN on CIFAR100. Again, it is clear that more classes benefited the CFNN, allowing more random behavior and higher Random Accuracy.}
    \label{fig:dropout:delta}
\end{center}
\vskip -0.2in
\end{figure*}

\label{appendix:mcdropout}
\subsection{MCDropout}

MCDropout \cite{gal2016mc-dropout} is a commonly used method which applies Dropout during inference to generate confidence intervals. Its formulation is different than ours in a few key ways, two of which are: we use majority while MCDropout calculates expectation; and we train our network with the express goal of improving accuracy using majority, while MCDropout does not address network training. 
Table \ref{tbl:dropout_results} emphasizes that the CFNNs framework can be applied to standard networks that use Dropout, and achieve substantial improvement in accuracy, while being very different from MCDropout in practice as well as in theory.

\section{Sampling Generalization - Additional Results}
\label{appendix:sampling_generalization}
Figure \ref{fig:n_samples_datasets} provide the empirical evidence for the effect of generalization sampling for our two CFNN implementations.

\begin{figure*}[t]
\vskip 0.2in
\begin{center}
    \begin{subfigure}[b]{0.35\textwidth}
         \centerline{\includegraphics[width=6cm]{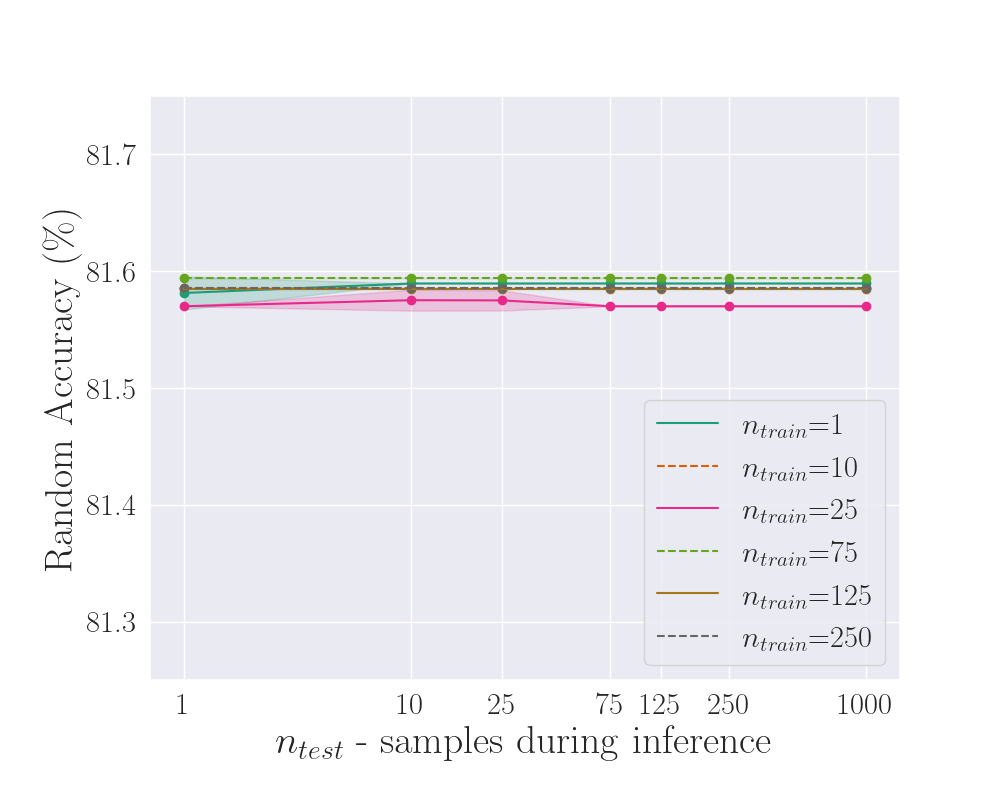}}
        \caption{}
        \label{fig:sampling_generalization:hyper_cifar10}
    \end{subfigure}
    \begin{subfigure}[b]{0.35\textwidth}
         \centerline{\includegraphics[width=6cm]{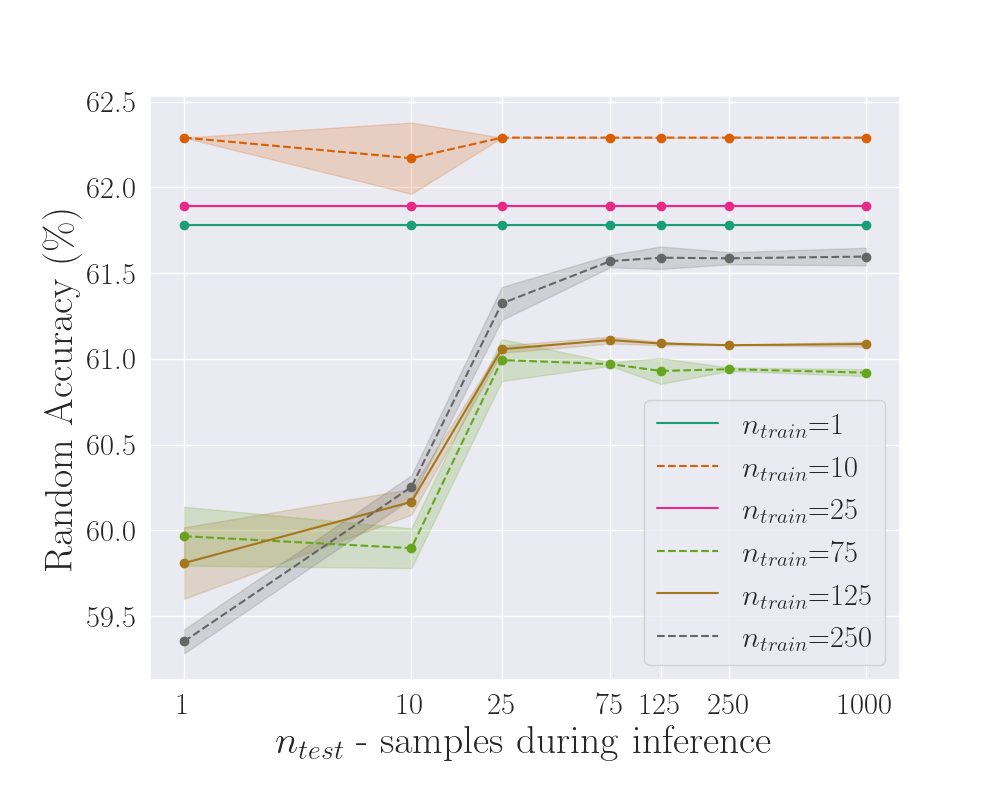}}
        \caption{}
        \label{fig:sampling_generalization:hyper_cifar100}
    \end{subfigure}
    \begin{subfigure}[b]{0.35\textwidth}
         \centerline{\includegraphics[width=6cm]{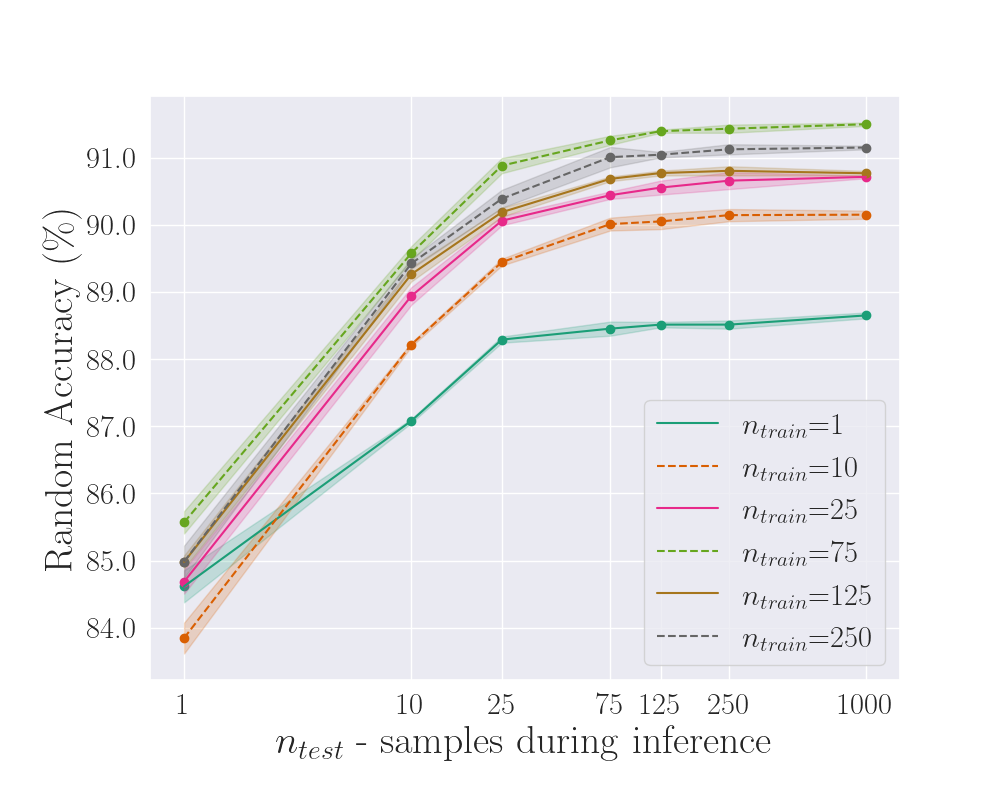}}
        \caption{}
        \label{fig:sampling_generalization:dropout_cifar10}
    \end{subfigure}
    \begin{subfigure}[b]{0.35\textwidth}
         \centerline{\includegraphics[width=6cm]{figures/sample generalization/n_dropout_cifar100_img.png}}
        \caption{}
        \label{fig:sampling_generalization:dropout_cifar100}
    \end{subfigure}
    \caption{\small Sampling Generalization for different CFNN architectures and datasets. Dropout CFNNs demonstrated this ability on both datasets, while our hypernetwork has shown it on CIFAR100 only, which is expected due to the model's centralized behavior on CIFAR10, as can be seen in Figure \ref{fig:hypernet:delta:cifar10}. Shaded areas indicate confidence bounds of 1 standard deviation from the expected value. (a) Hypernetwork CFNN on CIFAR10, with $D=6$. (b) Hypernetwork CFNN on CIFAR100, with $D=10$. (c) Dropout CFNN on CIFAR10. (d) Dropout CFNN on CIFAR100.}
    \label{fig:n_samples_datasets}
\end{center}
\vskip -0.2in
\end{figure*}

\end{document}